%% file: main.tex
\documentclass[11pt,letterpaper]{article}

\usepackage[margin=1in]{geometry}
\usepackage[T1]{fontenc}
\usepackage{lmodern}
\usepackage{microtype}
\input{preamble}

\title{Finite-Sample Guarantees for Fully Discrete\\Conditional Diffusion Models}
\author{Qiuyun Zhu\\[4pt]
  \small Department of Statistics and Actuarial Science\\
  \small University of Iowa\\
  \small\href{mailto:qiuyun-zhu@uiowa.edu}{\texttt{qiuyun-zhu@uiowa.edu}}}
\date{}

\hypersetup{
  hidelinks,
  pdftitle={Finite-Sample Guarantees for Fully Discrete Conditional Diffusion Models},
  pdfauthor={Qiuyun Zhu},
  pdfsubject={Conditional diffusion, finite-sample guarantees, and discrete sampling}
}

\begin{document}
\maketitle
\input{paper}
\end{document}

%% file: preamble.tex
\usepackage{amsmath,amssymb,amsfonts,amsthm}
\usepackage[round,authoryear]{natbib}
\usepackage[hidelinks]{hyperref}

\newcommand{\rev}[1]{#1}
\DeclareMathOperator*{\argmin}{arg\,min}

\theoremstyle{plain}
\newtheorem{theorem}{Theorem}
\newtheorem{proposition}[theorem]{Proposition}
\newtheorem{lemma}[theorem]{Lemma}

\newtheorem{assumption}{Assumption}

\theoremstyle{definition}
\newtheorem{definition}[theorem]{Definition}
\newtheorem{remark}{Remark}

%% file: paper.tex
\newcommand{\R}{\mathbb R}
\newcommand{\mx}{\mathbf x}
\newcommand{\mz}{\mathbf z}
\newcommand{\my}{\mathbf y}
\newcommand{\ms}{\mathbf s}
\newcommand{\mI}{\mathbf I}
\newcommand{\cF}{\mathcal F}
\newcommand{\cG}{\mathcal G}
\newcommand{\cH}{\mathcal H}
\newcommand{\cN}{\mathcal N}
\newcommand{\cE}{\mathcal E}
\newcommand{\KLdiv}{\operatorname{KL}}
\newcommand{\Var}{\operatorname{Var}}
\newcommand{\ind}{\mathbf 1}
\newcommand{\dd}{\,\mathrm d}

\begin{abstract}
We establish finite-sample guarantees for conditional diffusion models learned from data and implemented with finitely many denoising steps. Our approach aligns the score-training objective with the sampling grid, allowing statistical estimation and numerical discretization errors to be analyzed together. Under a smooth conditional density model, we derive high-probability bounds on the generated distribution's error and quantify the required number of score evaluations. For every smoothness index $\beta>0$, a suitable grid yields averaged conditional Kullback--Leibler (KL) divergence $\widetilde O_p(n^{-2\beta/(d_x+d_y+2\beta)})$ relative to the original target \rev{with or without early stopping}. Both the full and early-stopped samplers attain average target total-variation (TV) error $\widetilde O_p(n^{-\beta/(d_x+d_y+2\beta)})$. Here $n$ is the sample size, $d_x,d_y$ are the data and conditioning dimensions, and errors are averaged over the conditioning variable. The total-variation bound matches the rate established in existing conditional distribution-estimation theory, up to logarithmic factors.
\end{abstract}

\input{introduction}

\section{Problem Setup}\label{sec:model}

Let $(\mx_0,\my)$ be a random pair with $\mx_0\in\R^{d_x}$ and $\my\in\mathcal Y$, where $\mathcal Y\subset\R^{d_y}$ is compact. For each $\my$, let $q_0(\cdot\mid\my)$ denote the conditional density of the data.

\subsection{Forward and Reverse Processes}

Our analysis adopts the Ornstein--Uhlenbeck (OU) forward diffusion
\begin{equation}
 \dd\mx_t=-\mx_t\dd t+\sqrt{2}\,\dd\omega_t,
 \qquad \mx_0\sim q_0(\cdot\mid\my),
 \label{eq:forward}
\end{equation}
where $\{\omega_t\}_{t\ge0}$ is a standard Brownian motion. Conditional on $\my$, the dynamics in \eqref{eq:forward} induce the marginal law $q_t(\cdot\mid\my)$ of $\mx_t$. Moreover, the explicit OU solution yields the Gaussian transition kernel
\begin{equation}
 q_{t\mid0}(\cdot\mid\mx_0)
 =\cN(\mu_t\mx_0,\sigma_t^2\mI_{d_x}),
 \qquad \mu_t:=e^{-t},\qquad \sigma_t^2:=1-e^{-2t}.
 \label{eq:transition}
\end{equation}
Thus, $q_t(\cdot\mid\my)$ converges to $\cN(\mathbf0,\mI_{d_x})$ as $t\to\infty$.

Because the transition kernel is Gaussian, both score matching and initialization are tractable. Generating data, however, requires reversal of the forward dynamics. Accordingly, for a terminal time $T$, the time-reversed process $\mz_u:=\mx_{T-u}$ satisfies
\begin{equation*}
 \dd\mz_u
 =\left[\mz_u+2\nabla\log q_{T-u}(\mz_u\mid\my)\right]\dd u
 +\sqrt{2}\,\dd\bar\omega_u,
 \qquad \mz_0\sim q_T(\cdot\mid\my),
\end{equation*}
\rev{where, conditional on $\my$, $\{\bar\omega_u\}_{0\le u\le T}$ is a standard $d_x$-dimensional Brownian motion with respect to the reverse-time filtration, and $\nabla\log q_t(\cdot\mid\my)$ denotes the conditional score. Here $u:=T-t$ increases from $0$ to $T$ as the process evolves from the noised law $q_T(\cdot\mid\my)$ toward the data law $q_0(\cdot\mid\my)$.}


The gradient $\nabla\log q_t(\mathbf{x}\mid\my)$ is taken with respect to $\mathbf{x}$, holding $\mathbf{y}$ fixed. With the exact score and initialization from $q_T(\cdot\mid\mathbf{y})$, the reverse process has marginal $q_{T-u}(\cdot\mid\mathbf{y})$ at reverse time $u$. In practice, we use a learned score, initialize from a standard Gaussian, and discretize the reverse dynamics. Our analysis bounds the resulting sampling error through these three contributions: score estimation, initialization, and time discretization.

\subsection{Discretization and Implementation}

The exact reverse process depends on the unknown score at every time. Therefore, an implementable algorithm replaces this score by a learned approximation evaluated on a finite grid. Let
\begin{equation*}
 t_0:=0<t_1<\cdots<t_N:=T,
 \qquad h_k:=t_k-t_{k-1}.
\end{equation*}
\rev{The sampler traverses the forward-time grid in reverse, from $t_N=T$ to $t_0=0$, with $\widehat\mx_k$ denoting its state at noise level $t_k$.} At each $t_k$, let $\widehat\ms(\mx,\my,t_k)$ be the estimated conditional score and define $\widehat\ms_{t_k}(\mx,\my):=\widehat\ms(\mx,\my,t_k)$. We discretize the reverse dynamics using an exponential integrator that integrates the linear component of the reverse drift exactly. The score is held fixed over each interval. Conditional on $\my$, initialize
\[
 \widehat\mx_T\sim\cN(\mathbf0,\mI_{d_x}),
\]
and set
\begin{equation}
 \begin{aligned}
 \widehat\mx_{t_{k-1}}
 &:=e^{h_k}\widehat\mx_{t_k}
 +2(e^{h_k}-1)\widehat\ms_{t_k}(\widehat\mx_{t_k},\my)
 +\sqrt{e^{2h_k}-1}\,\boldsymbol\xi_k,\\
 \boldsymbol\xi_k&\stackrel{\mathrm{i.i.d.}}\sim\cN(\mathbf0,\mI_{d_x}),
 \qquad k=N,N-1,\ldots,1,
 \end{aligned}
 \label{eq:exp-integrator}
\end{equation}
\rev{Note that the Gaussian noises \( \{\boldsymbol{\xi}_{k} \}_{k = 1}^N \) are independent of the initialization and training data. The full sampler performs all $N$ steps and returns $\widehat\mx_0$ at forward time $0$, while early stopping omits the final step and returns $\widehat\mx_1$ at forward time $t_1$. Generating one sample requires $N$ score evaluations for the full sampler and $N-1$ for its early-stopped version. Given $\my$, we let $\widehat p_{T-t}(\cdot\mid\my)$ denote the conditional distribution of $\widehat\mx_t$} at forward time $t$. 
Equivalently, for each $k = 1, \ldots, N$, on the reverse-time interval $[T-t_k,T-t_{k-1}]$, the process solves
\[
 \dd\widehat\mz_u=
 \left[\widehat\mz_u+2\widehat\ms_{t_k}
 (\widehat\mz_{T-t_k},\my)\right]\dd u
 +\sqrt{2}\,\dd\bar\omega_u.
\]
\rev{Both the early-stopped and full samplers aim to recover the target conditional law $q_0(\cdot\mid\my)$. }

\section{Learning the Conditional Score from Data}\label{sec:score}

\rev{We first recall the conditional-score approximation under our density assumptions, and then analyze a denoising objective on the sampling grid.}

\rev{\subsection{Score Approximation}\label{sec:approximation}}

\rev{This subsection specifies the conditional density and network classes and recalls the score-approximation result of \citet{Fu2024UnveilCD}, which provides the score estimation rate in our statistical analysis.}

\rev{We first specify the smoothness class. For a multi-index $\alpha=(\alpha_1,\ldots,\alpha_d)$ of nonnegative integers, write $\|\alpha\|_1:=\sum_{j=1}^d\alpha_j$ and $\partial^\alpha:=\partial_1^{\alpha_1}\cdots\partial_d^{\alpha_d}$.}

\begin{definition}[H\"older class]
Let $\beta>0$, and set $r:=\lfloor\beta\rfloor$ and $\eta:=\beta-r\in[0,1)$. For $f:\R^d\to\R$, define
\[
 \|f\|_{\cH^\beta}
 :=\max_{\|\alpha\|_1\le r}\sup_x|\partial^\alpha f(x)|
 +\ind\{\eta>0\}
 \max_{\|\alpha\|_1=r}\sup_{x\ne z}
 \frac{|\partial^\alpha f(x)-\partial^\alpha f(z)|}
 {\|x-z\|_\infty^\eta},
\]
where the associated H\"older ball of radius $B$ is denoted by $\cH^\beta(\R^d,B)$.
\end{definition}

\begin{assumption}[Conditional density model]\label{ass:density}
 Fix a smoothness index $\beta>0$ and assume there exist a constant $R_y>0$, a compact set $\mathcal Y\subset[-R_y,R_y]^{d_y}$, constants $C_1,C_2,B>0$, and a function $f$ on $\R^{d_x}\times\mathcal Y$ such that
\begin{equation}
 f(\mx,\my)\ge C_1,
 \qquad
 q_0(\mx\mid\my)
 =\exp\!\left(-\frac{C_2}{2}\|\mx\|_2^2\right)f(\mx,\my).
 \label{eq:density-model}
\end{equation}
Moreover, $f$ admits an extension $\bar f\in\cH^\beta(\R^{d_x+d_y},B)$ satisfying $\bar f=f$ on $\R^{d_x}\times\mathcal Y$ and $\bar f\ge C_1$ on $\R^{d_x+d_y}$.
\end{assumption}


Assumption 1 implies
\(
C_1 e^{-C_2\|\mathbf{x}\|_2^2/2}
\le q_0(\mathbf{x}\mid\mathbf{y})
\le B e^{-C_2\|\mathbf{x}\|_2^2/2}
\),
uniformly over \(\mathbf{x}\in\mathbb{R}^{d_x}\) and \(\mathbf{y}\in\mathcal{Y}\).
As a result, all conditional moments of \(\mathbf{x}_0\) are uniformly bounded in \(\mathbf{y}\).
Joint H\"older smoothness of $\bar f$ in \((\mathbf{x},\mathbf{y})\) accounts for the dimension \(D:=d_x+d_y\) in the approximation rate, and no density assumption is imposed on the marginal distribution of \(\mathbf{y}\).

Let $\cF(M_t,W,\kappa,L,S)$ be the class of vector-valued ReLU networks $\ms(\mx,\my,t)$ with depth $L$, maximal width $W$, at most $S$ nonzero parameters, parameter magnitudes bounded by $\kappa$, and the time-dependent output envelope
\[
 \sup_{\mx,\my}\|\ms(\mx,\my,t)\|_\infty\le M_t
\]
at every time $t$ in the approximation range. Each network takes $(\mx,\my,t)$ as input. We let $K$ denote the network complexity parameter and $N$ the number of reverse steps.

\rev{The following lemma records Theorem~3.4 of \citet{Fu2024UnveilCD} in a compact space $\mathcal Y$.}

\rev{\begin{lemma}[Conditional-score approximation; {\citealp{Fu2024UnveilCD}}]\label{lem:approximation}
Suppose Assumption~\ref{ass:density} holds, and fix constants $C_-,C_+>0$. For sufficiently large $K$, there exists $\widetilde\ms\in\cF(M_t,W,\kappa,L,S)$ such that, uniformly over $\my\in\mathcal Y$ and $t\in[K^{-C_-},C_+\log K]$,
\begin{equation}
 \mathbb E_{q_t(\cdot\mid\my)}
 \|\widetilde\ms_t(\mx_t,\my)-\nabla\log q_t(\mx_t\mid\my)\|_2^2
 \le C\frac{K^{-2\beta/D}(\log K)^{\beta+1}}{\sigma_t^2},
 \label{eq:pointwise-approx}
\end{equation}
where the network parameters satisfy
\begin{equation*}
 M_t\lesssim\frac{\sqrt{\log K}}{\sigma_t},\qquad
 W\lesssim K\log^7K,\qquad
  L\lesssim\log^4K,
\qquad
S\lesssim K\log^9K,
\qquad
\log\kappa\lesssim\log^4K.
\end{equation*}
All implicit constants in this lemma depend only on $d_x,d_y,\beta,C_1,C_2,B,R_y,C_-,C_+$.
\end{lemma}}

\rev{The uniformity in Lemma~\ref{lem:approximation} allows a single conditional network to approximate the score at every time on the sampling grid.}

\rev{\subsection{Denoising Score Matching and Statistical Guarantee}}

We now use the approximation result in Section~3.1 to analyze score estimation from a finite training sample. We introduce a denoising loss whose population excess risk equals the weighted squared score error. Combining the approximation bound in Lemma~\ref{lem:approximation} with a concentration argument then yields a finite-sample guarantee for the empirical risk minimizer.

\rev{To distinguish the unknown target from a candidate network, we denote the true conditional score at time $t$ by $\ms_t^0$, defined as}
\[
 \ms_t^0(\mx,\my):=\nabla\log q_t(\mx\mid\my).
\]
\rev{We measure a candidate score's discrepancy from $\ms^0$ by the weighted squared $L^2$ error $\rho^2(\ms,\ms^0)$, which sums the mean squared errors at grid times with weights equal to the reverse-step lengths:}
\begin{equation}
 \rho^2(\ms,\ms^0)
 :=\sum_{k=1}^Nh_k\,
 \mathbb E_{\mx_{t_k},\my}
 \|\ms_{t_k}(\mx_{t_k},\my)-\ms_{t_k}^0(\mx_{t_k},\my)\|_2^2.
 \label{eq:rho}
\end{equation}
\rev{The expectation in \eqref{eq:rho} averages over the noisy data $\mx_{t_k}$ and the marginal distribution of $\my$. }

\rev{The unknown true score $\ms^0$ prevents direct minimization of the error $\rho^2(\ms,\ms^0)$. Instead, we use the denoising loss for an observed pair $(\mx_0,\my)$:}
\begin{equation}
 \ell(\mx_0,\my;\ms)
 :=\sum_{k=1}^Nh_k\,
 \mathbb E_{\mx_{t_k}\mid\mx_0}\bigl\|
 \ms_{t_k}(\mx_{t_k},\my)
 -\nabla\log q_{t_k\mid0}(\mx_{t_k}\mid\mx_0)
 \bigr\|_2^2.
 \label{eq:loss}
\end{equation}
The denoising loss $\ell(\mx_0,\my;\ms)$  compares the estimated score $\ms_{t_k}$ with the known Gaussian transition score $\nabla\log q_{t_k\mid0}$, averaging over artificial Gaussian noise introduced in \eqref{eq:transition}, at each $k$.
\rev{Conditioning on $(\mx_{t_k},\my)$ shows that $\nabla\log q_{t_k\mid0}$ has conditional mean $\ms^0_{t_k}(\mx_{t_k},\my)$. Consequently, the excess risk is exactly the weighted score error:}
\begin{equation}
 \mathbb E\ell(\mx_0,\my;\ms)-
 \mathbb E\ell(\mx_0,\my;\ms^0)
 =\rho^2(\ms,\ms^0).
 \label{eq:denoising-identity}
\end{equation}
\rev{For $n$ i.i.d. training pairs $\{(\mx_0^i,\my^i)\}_{i=1}^n$, let $\widehat{\mathcal L}_n$ be the empirical average loss. The fitted score $\widehat\ms$ is any measurable minimizer over the network class $\cF$:}
\begingroup
\thinmuskip=2mu
\medmuskip=2mu
\thickmuskip=3mu
\begin{equation}
 \widehat\ms\in\argmin_{\ms\in\cF}
 \widehat{\mathcal L}_n(\ms),
 \quad
 \widehat{\mathcal L}_n(\ms)
 :=\frac1n\sum_{i=1}^n\ell(\mx_0^i,\my^i;\ms).
 \label{eq:erm}
\end{equation}
\endgroup
The objective in \eqref{eq:erm} is empirical in the observed pairs but takes the exact expectation over the artificial noise. Note from the forward kernel:
\[
 \nabla\log q_{t_k\mid0}(\mx_{t_k}\mid\mx_0)
 =-\frac{\mx_{t_k}-\mu_{t_k}\mx_0}{\sigma_{t_k}^2}.
\]
\rev{Thus the empirical loss $\widehat{\mathcal L}_n$ is available without evaluating the unknown conditional density or its score.}

\rev{Next, we will introduce a bound on the score estimation error $\rho(\widehat\ms,\ms^0)$. As we shall see, the bound depends on the grid through the deterministic factor $$	A_n:=\sum_{k=1}^N\frac{h_k}{\sigma_{t_k}^2},$$ which  captures the greater sensitivity of score estimation at small noise levels. The following lemma provides both lower and upper bounds on $A_n$.}

\begin{lemma}[Bounds on the grid factor $A_n$]
	\label{lem:grid-factor}
	Let \(0=t_0<t_1<\cdots<t_N=T\), with step lengths
	\(h_k=t_k-t_{k-1}\) and noise variances
	\(\sigma_{t_k}^2=1-e^{-2t_k}\).
	Then
	\begin{equation}
		\begin{aligned}
			\max\{T,1/2\}
			&\le A_n
			\le C\bigl[1+T+\log^+(1/t_1)\bigr],
		\end{aligned}
		\label{eq:A-endpoints}
	\end{equation}
	where \(\log^+(a):=\max\{0,\log a\}\) and \(C\) is a
	universal constant independent of the grid.
	In particular, if the positive grid times lie in the
	approximation window \([K^{-C_-},C_+\log K]\) of
	Lemma~\ref{lem:approximation}, then 
	\[
	A_n\asymp\log K.
	\]
\end{lemma}

\emph{Proof:} See Appendix~\ref{sec:proof-score}.


\rev{The following theorem bounds the score estimation error in terms of the sample size, network complexity, and grid factor $A_n$.}

\begin{theorem}[Score-estimation rate]\label{thm:score}
Suppose Assumption~\ref{ass:density} holds and \rev{the positive grid times $t_1,\ldots,t_N$} lie in the approximation range of \rev{Lemma~\ref{lem:approximation}}. For sufficiently large $n$, \rev{choose the network class whose parameters satisfy the bounds in Lemma~\ref{lem:approximation}, with complexity parameter $K$ satisfying}
\begin{equation}
 K\asymp n^{D/(D+2\beta)}.
 \label{eq:K-choice}
\end{equation}
\rev{Define }
\begin{equation*}
 m:=\max\left\{9,(\beta+1)/2\right\} \quad\mbox{and}\quad
 \varepsilon_n
 :=C A_n^{1/2}n^{-\beta/(D+2\beta)}\log^m n.
\end{equation*}
For every fixed $M>0$, there exist constants $C,c>0$, depending only on $M$ and the fixed model parameters, we have
\begin{equation}
 \mathbb P\{\rho(\widehat\ms,\ms^0)>\varepsilon_n\}
 \le
 4\exp\!\left\{-c\frac{n^{D/(D+2\beta)}\log^{2m-2}n}{A_n}\right\}
 +C nK^{-M},
 \label{eq:score-tail}
\end{equation}
where probability is taken over the $n$ observed training pairs.
\rev{Therefore by Lemma~\ref{lem:grid-factor}, within the approximation window \([K^{-C_-},C_+\log K]\),  we have
\begin{equation}
 \rho(\widehat\ms,\ms^0)
 =O_p\!\left(n^{-\beta/(D+2\beta)}
 \log^{m+1/2}n\right).
 \label{eq:score-op}
\end{equation}}
\end{theorem}

\rev{\paragraph{Proof Idea.}
We first bound how accurately a single network can approximate the conditional score over the sampling grid (cf.\ Lemma~\ref{lem:approximation}). We next establish a concentration bound for the complete weighted denoising loss, accounting for the reuse of each training pair across all grid times. Combining these results with empirical risk minimization and the denoising identity \eqref{eq:denoising-identity} bounds the learned score error by the sum of the network approximation error and the finite-sample estimation error. Finally, balancing these errors through the network size in \eqref{eq:K-choice} gives the probability bound \eqref{eq:score-tail}, and the grid-factor bound (cf.\ Lemma~\ref{lem:grid-factor}) yields the rate \eqref{eq:score-op}. For the full proof, see Appendix~\ref{sec:proof-score}.}

\section{From Score Learning to Discrete Sampling}\label{sec:distribution}
We use Theorem~\ref{thm:score} to derive distributional guarantees for the full and early-stopped samplers. We first decompose the KL error into several components, then we choose a sampling grid that balances the discretization error with the number of score evaluations.

\subsection{KL Error Bound Decomposition}\label{sec:grid}
Consider any sampling grid \(0=t_0<t_1<\cdots<t_N=T\), with \(N\ge2\) and \(T\ge1\). We have the following lemma.

\begin{lemma}[KL error bound decomposition]\label{lem:chen-grid}\label{lem:arbitrary-grid}
 Fix the training sample and suppose that the fitted scores
 \(\{\widehat\ms_{t_k}\}_{k=1}^N\) are bounded and measurable. Then  we have the following:

\noindent\textup{(a)}
For \(u\in\{T-t_1,T\}\),
\begin{equation}
	\begin{aligned}
		\mathbb E_{\my}\KLdiv(q_{T-u}\|\widehat p_u)
		&\le
		\underbrace{
			\mathbb E_{\my}\KLdiv\!\left(
			q_T\middle\|\mathcal N(\mathbf0,\mI_{d_x})
			\right)}_{\textrm{initialization error}} +\underbrace{2\rho^2(\widehat\ms, \ms^0)}_{\textrm{score-estimation error}}\\
		&\quad+\underbrace{2\begin{aligned}
				&\sum_{k=1}^N\int_{t_{k-1}}^{t_k}
				\mathbb E\|\ms^0_t(\mx_t,\my)
				-\ms^0_{t_k}(\mx_{t_k},\my)\|_2^2
				\,\mathrm dt
		\end{aligned}}_{\textrm{discretization error}}.
	\end{aligned}
	\label{eq:general-decomposition}
\end{equation}

\medskip
\noindent\textup{(b)} 
In addition, assume the conditional density has the form
\[
q_0(\mx\mid\my)=e^{-C_2\|\mx\|_2^2/2}f(\mx,\my),
\quad C_1\le f(\mx,\my)\le B,
\]
for some constants \(C_1,C_2,B>0\), and let $$M_*:=\max_{k\ge2}\sup_{\mx,\my}\|\widehat\ms_{t_k}(\mx,\my)\|_2.$$ For sufficiently small \(0<t_1\le \bar t_1\),
\begin{equation}
 \mathbb E_{\my}\KLdiv(q_0\|\widehat p_{T-t_1})\le C\mathbb E_{\my}\KLdiv(q_{t_1}\|\widehat p_{T-t_1})
 +C\mathbb E_{\my}\chi^2(q_0\|q_{t_1})+C(1+T)(1+M_*^2)e^{-c/t_1}.
\label{eq:KL-early-general}
\end{equation}
Here the constants \(C,c>0\) and the threshold 
$\bar t_1$ depend only on \(C_1,C_2,B\) and \(d_x\), but not on the interior grid or the training sample.
\end{lemma}

In fact, Part~(a) follows from Proposition~C.3 of \citet{chen2023improved}, with \(u=T\) obtained by lower semicontinuity. Part~(b) results from our new bound on the KL error between \(q_{t_1}\) and \(q_0\), which is based on a localized density-ratio argument and a grid-independent lower bound on the sampler density. Note that Lemma~\ref{lem:chen-grid} requires no smoothness of \(f\) as in Assumption~\ref{ass:density}. That said, smoothness is needed when we bound the small-time error, as shown below.

\begin{lemma}[Small-time error]\label{lem:small-time-main}
Suppose Assumption~\ref{ass:density} holds with \(\beta>0\), and set \(r_\beta:=\min\{2,\beta\}\). For sufficiently small \(t_1>0\), we have 
\begin{equation}
 \int_0^{t_1}\mathbb E\|\ms^0_t(\mx_t,\my)-\ms^0_{t_1}(\mx_{t_1},\my)\|_2^2\dd t
 +\mathbb E_{\my}\chi^2(q_0\|q_{t_1})\le Ct_1^{r_\beta}.
 \label{eq:small-time-main}
\end{equation}
\end{lemma}

\begin{proof}
See Appendix~\ref{sec:proof-kl}. 
\end{proof}


\paragraph{Geometric--linear sampling scheme.}
From~\citet{chen2023improved}, we know that for arbitrary sampling grid, 
the discretization error from $k=2$ can be bounded by \(Cd_x^2\Pi\), where
\[
 \Pi:=\sum_{k=2}^N\frac{h_k^2}{\sigma_{t_{k-1}}^4}.
\]
For sufficiently large \(n\), choose
\begin{equation}
 t_1:=n^{-\frac{2\beta}{(D+2\beta)r_\beta}},\qquad
 c_n:=c_0n^{-\frac{2\beta}{D+2\beta}},\qquad
 T:=\frac{2\beta}{D+2\beta}\log n,
\label{eq:parameters}
\end{equation}
where \(c_0>0\) is a fixed sufficiently small constant. Starting at \(t_1\), set
\begin{equation}
 t_k=\begin{cases}
 \min\{T,t_{k-1}/(1-c_n)\},&t_{k-1}<1-c_n,\\
 \min\{T,t_{k-1}+c_n\},&t_{k-1}\ge1-c_n,
 \end{cases}
 \label{eq:grid-rule}
\end{equation}
and stop at \(t_N=T\). Except for a shortened last interval, the steps satisfy \(h_k=c_n\min\{t_k,1\}\), giving geometric spacing near zero and linear spacing thereafter.

Following \citet
{chen2023improved}, the sampling schedule in~\eqref{eq:grid-rule} minimizes \(\Pi\) up to constants for fixed endpoints and number of retained steps, which gives  
\begin{equation}
 \Pi\asymp\frac{\{\log(1/t_1)+T\}^2}{N-1}
 \lesssim c_n\left(\log\frac1{t_1}+T\right).
 \label{eq:Pi-order-optimal}
\end{equation}

Now, combining the Lemmas~\ref{lem:chen-grid} and~\ref{lem:small-time-main} with this sampling grid 
yields the following proposition.

\begin{proposition}[KL error bound on the geometric--linear grid]\label{prop:sampling-bound}
Under the setting of Theorem~\ref{thm:score}, let \(C_->2\beta/(Dr_\beta)\) and \(C_+>2\beta/D\). On the sampling grid 
\eqref{eq:grid-rule}, for sufficiently large \(n\), we have 
\begin{equation}
 \max_{v\in\{T-t_1,T\}}\mathbb E_{\my}\KLdiv(q_0\|\widehat p_v)
 \lesssim e^{-T}+\rho^2+t_1^{r_\beta}
       +d_x^2c_n\left(\log\frac1{t_1}+T\right).
\label{eq:schematic-decomposition}
\end{equation}
\end{proposition}

\noindent\emph{Proof idea.}
The initialization error is bounded by \(Ce^{-T}\). 
The network envelope gives \(M_*^2\lesssim(\log n)/t_1\), so the exponential remainder in Lemma~\ref{lem:arbitrary-grid}(b) is \(o(t_1^{r_\beta})\). Apply Lemmas~\ref{lem:chen-grid} and~\ref{lem:small-time-main}, bound the retained score-estimation error by \(\rho^2\), and insert \eqref{eq:Pi-order-optimal}. Appendix~\ref{sec:proof-sampling-bound} gives the details. The score-estimation theorem then provides \(\rho^2\) on the same training-sample event for both outputs.

\rev{\subsection{Finite-Sample KL Guarantees With and Without Early Stopping}}

\rev{We now combine Theorem~\ref{thm:score} with Proposition~\ref{prop:sampling-bound} on the sampling grid in~\eqref{eq:grid-rule}. 
The following theorem gives the same KL error bound for both outputs from the full sampler and its early stopped version. } 

\rev{\begin{theorem}[Convergence with and without early stopping]\label{thm:kl}
Suppose Assumption~\ref{ass:density} holds with any fixed $\beta>0$. 
Use the sampling grid in~\eqref{eq:grid-rule}.
Under the setting of Theorem~\ref{thm:score} with $C_->2\beta/(Dr_\beta)$ and $C_+>2\beta/D$,  
we have 
\begin{equation}
 N=O\!\left(n^{\frac{2\beta}{D+2\beta}}\log n\right)
 \label{eq:N-rate}
\end{equation}
In addition, for every fixed $M>0$, there exist constants $C,c>0$ such that 
with probability at least
\begin{equation}
 1-4\exp\{-c n^{D/(D+2\beta)}
 \log^{2m-3}n\}-CnK^{-M},
 \label{eq:kl-prob}
\end{equation}
we have 
\begin{equation}
 \max_{u\in\{T-t_1,T\}}\mathbb E_{\my}\KLdiv\!\left(q_0(\cdot\mid\my)
 \,\middle\|\,\widehat p_u(\cdot\mid\my)\right)
 \le Cn^{-\frac{2\beta}{D+2\beta}}\log^{2m+1}n,
 \label{eq:kl-highprob}
\end{equation}
where $2m+1=\max\{19,\beta+2\}$.
\end{theorem}}

\rev{\paragraph{Proof Idea.}
We first obtain a finite-sample bound on the learned score (cf.\ Theorem~\ref{thm:score}). We then use the sampling error decomposition to pass from score accuracy to KL error control relative to the original conditional data distribution, for both the full and early-stopped samplers (cf.\ Proposition~\ref{prop:sampling-bound}). Finally, the grid calibration \eqref{eq:parameters} keeps initialization, small-time, and discretization errors within the statistical error bound while giving the explicit step count. Combining these results proves the common KL error bound for both outputs with the stated probability. For the full proof, see Appendix~\ref{sec:proof-kl-theorem}.}


\begin{remark}[TV error \rev{bound}]
\rev{Applying Pinsker's inequality conditionally on $\my$ and then Jensen's inequality to Theorem~\ref{thm:kl} gives, for both stopping rules,
\[
 \max_{u\in\{T-t_1,T\}}\mathbb E_{\my}\operatorname{TV}
 \!\left(q_0(\cdot\mid\my),\widehat p_u(\cdot\mid\my)\right)
 =O_p\!\left(n^{-\beta/(D+2\beta)}\log^{m+1/2}n\right).
\]}
\rev{Up to logarithmic factors, the rate matches those of \citet{Fu2024UnveilCD} and \citet{tian2026} under similar assumptions. It also matches that in \citet{Tang2025Conditional} when $q_0(\mx|\my)$ has the same H\"older smoothness in $\mx$ and $\my$. Crucially, these statistical guarantees concern
continuous-time conditional diffusion models, whereas our bound applies to the fully discrete sampler.
Thus, our result explicitly accounts for time-discretization errors while retaining the same statistical rate,
for both the early-stopped and full sampler.}
\end{remark}

%
%
%

\input{discussion}

\subsection*{AI Use Statement}

Generative AI tools were used to assist with literature search, language editing, manuscript organization, LaTeX formatting, and the development and checking of mathematical arguments. The authors independently reviewed the cited primary sources and verified every mathematical statement and proof. The authors take full responsibility for the final content of this work, including all text and mathematical claims produced with the aid of generative AI.

\bibliographystyle{apalike}
\bibliography{ref}

\clearpage
\phantomsection
\addcontentsline{toc}{section}{Appendix}
\begin{center}
{\huge\bfseries Appendix\par}
\end{center}
\appendix

\rev{\section{Proof of Theorem~\ref{thm:score}}\label{sec:proof-score}}

\rev{The proof has four parts: score and loss envelopes, localization and covering entropy, a finite-cover concentration inequality, and its application to the fitted score. All constants are uniform in the conditioning variable and the number of grid points; they may depend on the fixed model and network parameters and the prescribed tail exponent.}

\rev{\subsection{Proof of Lemma~\ref{lem:approximation}}}

\rev{We obtain Lemma~\ref{lem:approximation} from Theorem~3.4 of \citet{Fu2024UnveilCD}, which is stated for $\my\in[0,1]^{d_y}$. The compact set $\mathcal Y$ in Assumption~\ref{ass:density} is contained in $[-R_y,R_y]^{d_y}$. The extended factor need not define a normalized density outside $\mathcal Y$. Define its normalizing function by
\[
 Z(\my):=\int_{\R^{d_x}}e^{-C_2\|\mx\|_2^2/2}\bar f(\mx,\my)\dd\mx.
\]
The bounds on $\bar f$ give $C_1(2\pi/C_2)^{d_x/2}\le Z\le B(2\pi/C_2)^{d_x/2}$, and $Z=1$ on $\mathcal Y$. Gaussian integration preserves the required H\"older regularity, so $\bar f/Z$ is positive and H\"older with fixed bounds and defines normalized conditional densities on the containing cube. These densities agree with $q_0$ on $\mathcal Y$. Rescaling this cube to $[0,1]^{d_y}$ and applying Theorem~3.4 of \citet{Fu2024UnveilCD} yields Lemma~\ref{lem:approximation}, with constants also depending on $R_y$. The same construction permits the cited small-time estimate used below.}

\subsection{Basic Score and Tail Bounds}

Localization and concentration rely on deterministic and moment bounds. The Gaussian envelope in Assumption~\ref{ass:density} yields Gaussian tail bounds, while the score exhibits inverse-variance growth.

\begin{lemma}[Gaussian tails and score envelope]\label{lem:tails-score}
Under Assumption~\ref{ass:density}, for every integer $j\ge1$ there are constants $c,C_j>0$ such that, uniformly over $\my$,
\begin{equation}
 \mathbb E[\|\mx_0\|_2^{2j}\ind\{\|\mx_0\|_\infty>R\}\mid\my]
 \le \rev{C_je^{-cR^2}}.
 \label{eq:x0-tail}
\end{equation}
Moreover, for every $t>0$,
\begin{equation}
 \|\nabla\log q_t(\mx\mid\my)\|_2
 \le C\left(\|\mx\|_2+\frac1{\sigma_t}\right).
 \label{eq:score-envelope}
\end{equation}
\end{lemma}

\begin{proof}
The H\"older norm bounds $f$ from above. Hence \eqref{eq:density-model} implies, uniformly in $\my$,
\[
 q_0(\mx\mid\my)\le B\exp(-C_2\|\mx\|_2^2/2).
\]
\rev{For $R\ge0$, let $A_R:=\{\mx:\|\mx\|_\infty>R\}$. On this set, $\|\mx\|_2\ge\|\mx\|_\infty>R$, so
\[
 e^{-C_2\|\mx\|_2^2/2}
 =e^{-C_2\|\mx\|_2^2/4}e^{-C_2\|\mx\|_2^2/4}
 \le e^{-C_2R^2/4}e^{-C_2\|\mx\|_2^2/4}.
\]
Consequently,
\[
 \begin{aligned}
 &\mathbb E[\|\mx_0\|_2^{2j}\ind\{\|\mx_0\|_\infty>R\}\mid\my]\\
 &\quad\le B\int_{A_R}\|\mx\|_2^{2j}e^{-C_2\|\mx\|_2^2/2}\dd\mx\\
 &\quad\le Be^{-C_2R^2/4}
 \int_{\R^{d_x}}\|\mx\|_2^{2j}e^{-C_2\|\mx\|_2^2/4}\dd\mx
 =C_je^{-cR^2},
 \end{aligned}
\]
where $c:=C_2/4$ and the integral defining $C_j$ is a finite Gaussian moment. These constants are independent of $R$ and $\my$, proving \eqref{eq:x0-tail} without a polynomial prefactor.}
Lemma~B.9 of \citet{Fu2024UnveilCD}, expressed under the OU normalization \eqref{eq:forward}, gives the score estimate \eqref{eq:score-envelope}.
\end{proof}

For a network $\ms\in\cF$, let
\[
 R_{\ms}(\mx_0,\my)
 :=\sum_{k=1}^Nh_k\,
 \mathbb E_{\mx_{t_k}\mid\mx_0}
 \|\ms_{t_k}(\mx_{t_k},\my)-\ms^0_{t_k}(\mx_{t_k},\my)\|_2^2.
\]
Then $\mathbb E R_{\ms}=\rho^2(\ms,\ms^0)$.

The discrete-time denoising score-matching loss retains all cross-time covariance terms and therefore does not require an independence assumption across discretization times.

\begin{lemma}[Discrete-time loss inequalities]\label{lem:aggregate}
Let
\[
 \Delta_{\ms}(\mx_0,\my)
 :=\ell(\mx_0,\my;\ms)-\ell(\mx_0,\my;\ms^0).
\]
Uniformly over $\ms\in\cF$,
\begin{equation}
 |\Delta_{\ms}|^2
 \le2R_{\ms}\{\ell(\mx_0,\my;\ms)
 +\ell(\mx_0,\my;\ms^0)\},
 \label{eq:aggregate-cs}
\end{equation}
\begin{equation}
 \ell(\mx_0,\my;\ms)\le CA_n\log K,
 \label{eq:loss-network-envelope}
\end{equation}
and
\begin{equation}
 \ell(\mx_0,\my;\ms^0)
 \le CA_n(1+\|\mx_0\|_2^2).
 \label{eq:loss-truth-envelope}
\end{equation}
\end{lemma}

\begin{proof}
Set
$U_k:=\nabla\log q_{t_k\mid0}(\mx_{t_k}\mid\mx_0)$ and
$S_k^0:=\ms_{t_k}^0(\mx_{t_k},\my)$. Expanding the difference of two squares gives
\[
 \Delta_{\ms}
 =\sum_{k=1}^Nh_k\,
 \mathbb E_{\mx_{t_k}\mid\mx_0}
 \langle\ms_{t_k}-S_k^0,\ms_{t_k}+S_k^0-2U_k\rangle.
\]
Cauchy--Schwarz over the product of the discrete time measure and the conditional transition laws yields
\[
 |\Delta_{\ms}|^2
 \le R_{\ms}
 \sum_{k=1}^Nh_k\mathbb E
 \|\ms_{t_k}+S_k^0-2U_k\|_2^2,
\]
where the second factor is at most
$2\ell(\mx_0,\my;\ms)+2\ell(\mx_0,\my;\ms^0)$, proving \eqref{eq:aggregate-cs}. This argument is applied to the full sum over discretization times and therefore automatically includes all cross-time covariance terms.

From \eqref{eq:transition},
\[
 \mathbb E_{\mx_t\mid\mx_0}
 \|\nabla\log q_{t\mid0}(\mx_t\mid\mx_0)\|_2^2
 =\frac{d_x}{\sigma_t^2}.
\]
Using also $\|\ms_t\|_\infty\lesssim\sqrt{\log K}/\sigma_t$ proves \eqref{eq:loss-network-envelope}. For \eqref{eq:loss-truth-envelope}, use \eqref{eq:score-envelope},
$\mathbb E(\|\mx_t\|_2^2\mid\mx_0)
 =\mu_t^2\|\mx_0\|_2^2+d_x\sigma_t^2$, and
$\sum_kh_k\le A_n$.
\end{proof}

\subsection{Localization and Entropy}

The network-covering result applies on bounded input domains, whereas both $\mx_0$ and the noised variables $\mx_{t_k}$ are unbounded. Both sets of inputs are therefore truncated at logarithmic radii; tail bounds quantify the discarded regions, and a cover of the network class transfers to the loss class.

\rev{For the empirical-process arguments below, let $P$ be the joint law of one training pair and $P_n$ its empirical measure. For an integrable function $g$, write $Pg:=\mathbb E[g(\mx_0,\my)]$ and $P_ng:=n^{-1}\sum_{i=1}^n g(\mx_0^i,\my^i)$. The norm $\|g\|_{L_2(P)}:=(P|g|^2)^{1/2}$ measures population size, while $\|\cdot\|_\infty$ for a function denotes its supremum norm over the specified domain, using the maximum norm for vector-valued outputs. We write $\mathcal N(u,\mathcal G,\|\cdot\|)$ for the covering number at radius $u$, and $H_\infty(u,\mathcal G):=\log\mathcal N(u,\mathcal G,\|\cdot\|_\infty)$ for the sup-norm covering entropy.}

Fix an arbitrary $M>0$. Choose constants $C_0,C_x$ sufficiently large and set
\[
 R_0:=C_0\sqrt{\log K},
 \qquad R_x:=C_x\sqrt{\log K}.
\]
Define the localized loss
\begin{align*}
 \ell_R(\mx_0,\my;\ms)
 :={}&\ind\{\|\mx_0\|_\infty\le R_0\}
 \sum_{k=1}^Nh_k\,
 \mathbb E_{\mx_{t_k}\mid\mx_0}
 \big[\|\ms_{t_k}-\nabla\log q_{t_k\mid0}\|_2^2
 \ind\{\|\mx_{t_k}\|_\infty\le R_x\}\big].
\end{align*}

\begin{lemma}[Localization error]\label{lem:localization}
For every prescribed $M>0$, the constants $C_0,C_x$ can be chosen to satisfy
\begin{equation}
 \mathbb P(\|\mx_0\|_\infty>R_0)\le CK^{-M}.
 \label{eq:localization-prob}
\end{equation}
Moreover, uniformly over $\ms\in\cF$,
\begin{equation}
 \mathbb E|\ell(\mx_0,\my;\ms)-\ell_R(\mx_0,\my;\ms)|
 \le \eta_K,
 \qquad
 \eta_K\le CA_n\log K\,K^{-M},
 \label{eq:localization-mean}
\end{equation}
and the same bound holds with $\ms$ replaced by $\ms^0$. Finally, on
$\{\|\mx_0\|_\infty\le R_0\}$,
\begin{equation}
 0\le \ell(\mx_0,\my;\ms)-\ell_R(\mx_0,\my;\ms)
 \le r_K,
 \qquad
 r_K\le CA_n\log K\,K^{-M},
 \label{eq:localization-pointwise}
\end{equation}
uniformly over $\ms\in\cF$.
\end{lemma}

\rev{\begin{proof}
Choose $C_0$ large enough that \eqref{eq:x0-tail} gives both \eqref{eq:localization-prob} and the required moment-tail bound outside $R_0$. On $\{\|\mx_0\|_\infty\le R_0\}$, write $\mx_t=\mu_t\mx_0+\sigma_t\mathbf Z$, where $\mathbf Z\sim\cN(\mathbf0,\mI_{d_x})$. Since $\mu_t,\sigma_t\le1$, the event $\|\mx_t\|_\infty>R_x$ implies
\[
 \|\mathbf Z\|_\infty>(C_x-C_0)\sqrt{\log K}.
\]
The transition score is exactly $-\mathbf Z/\sigma_t$. The network envelope therefore gives
\[
 \|\ms_t(\mx_t,\my)-\nabla\log q_{t\mid0}(\mx_t\mid\mx_0)\|_2^2
 \le \frac{C}{\sigma_t^2}(\log K+\|\mathbf Z\|_2^2).
\]
The same bound holds for $\ms^0$, using \eqref{eq:score-envelope} and
$\|\mx_t\|_2^2\le C(\log K+\|\mathbf Z\|_2^2)$ on the localized data region. Choose $C_x>C_0$ sufficiently large that Gaussian tail integration gives
\[
 \mathbb E\!\left[(\log K+\|\mathbf Z\|_2^2)
 \ind\{\|\mathbf Z\|_\infty>(C_x-C_0)\sqrt{\log K}\}\right]
 \le C\log K\,K^{-M}.
\]
Multiplying by $h_k/\sigma_{t_k}^2$ and summing proves \eqref{eq:localization-pointwise}, and its true-score analogue, with bound $CA_n\log K\,K^{-M}$ independent of $N$. On $\{\|\mx_0\|_\infty>R_0\}$, the full-loss envelopes in Lemma~\ref{lem:aggregate} and the Gaussian moment-tail bound give the same population order. Combining the two regions proves \eqref{eq:localization-mean} for both the network and the true score.
\end{proof}}

Let
\[
 B_K:=\sqrt{\log K}\sum_{k=1}^N\frac{h_k}{\sigma_{t_k}}.
\]
\rev{For every grid in the approximation range, the endpoint estimate \eqref{eq:A-endpoints} and $0<\sigma_t\le1$ give
\[
 B_K\le\sqrt{\log K}\,A_n=O((\log K)^{3/2}).
\]
This bound is uniform in the number of grid points.}

Once localization is in place, the loss becomes Lipschitz in the network output on a bounded domain. Consequently, the sparse-network covering bound transfers to the excess-loss class for the discrete-time objective.

\rev{\begin{lemma}[Covering entropy of the localized loss class]\label{lem:entropy}
Let
\[
 \cG_R:=\{\ell_R(\cdot;\ms)-\ell_R(\cdot;\widetilde\ms):\ms\in\cF\}.
\]
For $0<u\le1$, its sup-norm covering entropy satisfies
\begin{equation}
 H_\infty(u,\cG_R)
 \le CK\log^{17}K
 +CK\log^9K\log\!\left(1+\frac{CB_K}{u}\right).
 \label{eq:entropy-bound}
\end{equation}
In particular, if $u^{-1}$ is bounded by a fixed power of $K$, then
\begin{equation}
 H_\infty(u,\cG_R)\le CK\log^{17}K.
 \label{eq:entropy-polynomial}
\end{equation}
\end{lemma}}

\rev{\begin{proof}
Suppose $\|\ms^1-\ms^2\|_\infty\le\delta$ on the localized input domain. The difference-of-squares identity and the transition-score moment imply
\begin{align*}
 |\ell_R(\mx_0,\my;\ms^1)-\ell_R(\mx_0,\my;\ms^2)|
 &\le C\delta\sum_{k=1}^Nh_k
 \left(M_{t_k}+\frac1{\sigma_{t_k}}\right)\\
 &\le CB_K\delta.
\end{align*}
Outside $\{\|\mx_0\|_\infty\le R_0\}$ both localized losses vanish, so this bound holds over the entire training-pair domain. Subtracting the fixed approximant's loss does not change distances. A network cover at radius $\delta:=\min\{1,u/(CB_K)\}$ therefore induces a sup-norm $u$-cover of $\cG_R$, after a harmless adjustment of $C$.

Lemma~D.8 of \citet{Fu2024UnveilCD} states that, on an input domain of radius $R$,
\[
 \log\cN(\delta,\cF,\|\cdot\|_\infty)
 \le CS\left\{
  L\log\kappa+(L+1)\log W
  +\log[L^2(WR+2)]+\log\frac1\delta
  \right\},
\]
where $R=O(\log K)$ after the time coordinate is included. If a covering ball for the larger parameter-bounded network class meets the prescribed output-envelope subclass, select a center from that intersection. Re-centering doubles the radius but does not increase the number of balls; hence the same entropy order holds with centers satisfying the output envelope used above. Substituting
$S\lesssim K\log^9K$, $L\lesssim\log^4K$,
$W\lesssim K\log^7K$, and $\log\kappa\lesssim\log^4K$ gives
\[
 \log\cN(\delta,\cF,\|\cdot\|_\infty)
 \le CK\log^{17}K+CK\log^9K\log\frac1\delta.
\]
Substituting the chosen $\delta$ proves \eqref{eq:entropy-bound}. If $u^{-1}$ is at most a fixed power of $K$, the bound on $B_K$ gives $\log(1+CB_K/u)=O(\log K)$. The second term is then dominated by the first, proving \eqref{eq:entropy-polynomial}. The constants depend on the fixed polynomial exponent, but not on $N$.
\end{proof}}

\rev{\subsection{A Localized Empirical-Risk Inequality}

We now combine the variance and covering estimates in a direct concentration argument. As in the risk-slicing approach of Proposition~1 of \citet{tian2026}, candidates are grouped by their population excess risk. Here bounded localized losses and a sup-norm cover allow scalar Bernstein inequalities and finite union bounds, with explicit constants. No bracketing maximal inequality is needed.

\begin{lemma}[Localized ERM bound by finite covers]\label{lem:localized-erm}
Let $P_n$ be the empirical measure of $n$ i.i.d.\ observations with law $P$. Suppose the full loss has nonnegative excess risk
$\mathcal R_\gamma:=P\{l(\gamma)-l(\gamma^0)\}$ for every $\gamma\in\Gamma$, and let $l_R$ be a localized loss. Fix an approximating element $\widetilde\gamma\in\Gamma$ and a measurable estimator $\widehat\gamma\in\Gamma$. For constants $r,\eta\ge0$, assume that, on an event $\mathcal A_n$,
\begin{equation}
 P_nl_R(\widehat\gamma)
 \le P_nl_R(\widetilde\gamma)+r,
 \label{eq:approx-local-erm}
\end{equation}
that
\begin{equation}
 \sup_{\gamma\in\Gamma\cup\{\gamma^0\}}
 |P\{l(\gamma)-l_R(\gamma)\}|\le\eta,
 \label{eq:population-local-gap}
\end{equation}
and write $\Delta_{R,\gamma}:=l_R(\gamma)-l_R(\gamma^0)$. Assume there are constants $c_v\ge1$ and $c_b>0$ such that, for every $\gamma\in\Gamma$,
\begin{equation}
 P\Delta_{R,\gamma}^2
 \le c_v\mathcal R_\gamma
 \label{eq:generic-variance}
\end{equation}
and the localized excess loss is uniformly bounded:
\begin{equation}
 \|\Delta_{R,\gamma}\|_\infty\le c_b.
 \label{eq:generic-boundedness}
\end{equation}
Let $H_\infty(u)$ denote the sup-norm covering entropy of
$\{l_R(\gamma)-l_R(\widetilde\gamma):\gamma\in\Gamma\}$. The constants $c_v,c_b$ may depend on $n$.

For any $\varepsilon>0$ satisfying
\begin{equation}
 \varepsilon^2\ge
 4\mathcal R_{\widetilde\gamma}+32\eta+8r
 \label{eq:localized-radius-condition}
\end{equation}
and
\begin{equation}
 H_\infty(\varepsilon^2/128)+\log2
 \le\frac{n\varepsilon^2}{80(c_v+c_b)},
 \label{eq:localized-entropy-condition}
\end{equation}
the estimator obeys
\begin{equation}
 \mathbb P\{\mathcal R_{\widehat\gamma}>\varepsilon^2,\mathcal A_n\}
 \le2\exp\!\left\{-\frac{n\varepsilon^2}{80(c_v+c_b)}\right\}.
 \label{eq:localized-erm-tail}
\end{equation}
\end{lemma}

\begin{proof}
We track separately the deterministic risk gap, finite-cover approximation, and stochastic deviation.

\textit{Step 1: a deviation is necessary on each risk range.}
For each integer $j\ge0$, set $s_j:=2^j\varepsilon^2$ and define the risk range
\[
 \mathcal S_j:=\{\gamma\in\Gamma:
 s_j\le\mathcal R_\gamma<2s_j\}.
\]
Write $G_\gamma:=l_R(\gamma)-l_R(\widetilde\gamma)$. The population localization error in \eqref{eq:population-local-gap} gives
\[
 PG_\gamma\ge\mathcal R_\gamma-\mathcal R_{\widetilde\gamma}-2\eta.
\]
Condition \eqref{eq:localized-radius-condition} implies
$\mathcal R_{\widetilde\gamma}\le\varepsilon^2/4$,
$2\eta\le\varepsilon^2/16$, and $r\le\varepsilon^2/8$.
Thus every $\gamma\in\mathcal S_j$ satisfies
\[
 PG_\gamma-r\ge s_j-\frac7{16}\varepsilon^2
 \ge\frac9{16}s_j.
\]
On $\mathcal A_n$, empirical optimality \eqref{eq:approx-local-erm} gives
$P_nG_{\widehat\gamma}\le r$. Consequently,
$\{\widehat\gamma\in\mathcal S_j,\mathcal A_n\}$ implies
$(P-P_n)G_{\widehat\gamma}\ge9s_j/16$.

\textit{Step 2: reduce to a finite cover with centers in the same risk range.}
Take a deterministic sup-norm cover of the full loss-difference class with radius $\varepsilon^2/128$. For each covering ball that meets $\{G_\gamma:\gamma\in\mathcal S_j\}$, select a representative from that intersection. The resulting finite set $\mathcal T_j$ covers this restricted class with radius $\varepsilon^2/64$ and satisfies
\[
 |\mathcal T_j|\le\exp\{H_\infty(\varepsilon^2/128)\}.
\]
Re-centering in this way ensures that every representative has risk below $2s_j$, as needed for the variance bound. Since
$|(P-P_n)(G-H)|\le2\|G-H\|_\infty$, the event from Step 1 implies that some $H\in\mathcal T_j$ satisfies
\[
 (P-P_n)H\ge\frac9{16}s_j-\frac{\varepsilon^2}{32}
 \ge\frac12s_j.
\]

\textit{Step 3: apply Bernstein's inequality with explicit constants.}
For a representative $H=G_\gamma$ with $\gamma\in\mathcal S_j$, the variance condition yields
\[
 \begin{aligned}
 \Var(H)&\le PH^2
 \le2P\Delta_{R,\gamma}^2+2P\Delta_{R,\widetilde\gamma}^2\\
 &\le2c_v(2s_j+\varepsilon^2/4)
 \le\frac92c_vs_j.
 \end{aligned}
\]
The envelope condition gives $\|H-PH\|_\infty\le4c_b$.
The one-sided scalar Bernstein inequality for the original i.i.d.\ observations therefore gives
\[
 \begin{aligned}
 \mathbb P\{(P-P_n)H\ge s_j/2\}
 &\le\exp\!\left\{
 -\frac{n(s_j/2)^2}
 {2\{(9/2)c_vs_j+4c_b(s_j/2)/3\}}\right\}\\
 &=\exp\!\left\{-\frac{ns_j}{36c_v+(16/3)c_b}\right\}
 \le\exp\!\left\{-\frac{ns_j}{40(c_v+c_b)}\right\}.
 \end{aligned}
\]
This probability is not conditioned on $\mathcal A_n$: the localization event is used only for the deterministic implication in Step 1, so the observations remain i.i.d.\ in the concentration calculation.

\textit{Step 4: sum over representatives and risk ranges.}
Set $\Lambda:=n\varepsilon^2/[80(c_v+c_b)]$. By \eqref{eq:localized-entropy-condition}, the covering entropy is at most $\Lambda$ and $\Lambda\ge\log2$. Finite union bounds and then a countable union bound give
\[
 \begin{aligned}
 \mathbb P\{\mathcal R_{\widehat\gamma}>\varepsilon^2,\mathcal A_n\}
 &\le\sum_{j\ge0}
   \exp\{H_\infty(\varepsilon^2/128)-2^{j+1}\Lambda\}\\
 &\le\sum_{j\ge0}e^{-2^j\Lambda}
 \le\frac{e^{-\Lambda}}{1-e^{-\Lambda}}
 \le2e^{-\Lambda}.
 \end{aligned}
\]
The penultimate inequality uses $2^j\ge j+1$. All numerical constants are explicit and remain valid when $c_v$ and $c_b$ grow with $n$, proving \eqref{eq:localized-erm-tail}.
\end{proof}}

\rev{For the present model, the generic excess risk is $\mathcal R_{\ms}=\rho^2(\ms,\ms^0)$ by the denoising identity. Define the localized excess loss by}
\[
 \Delta_{R,\ms}(\mx_0,\my)
 :=\ell_R(\mx_0,\my;\ms)-\ell_R(\mx_0,\my;\ms^0).
\]
\rev{The discrete-time Cauchy--Schwarz argument of Lemma~\ref{lem:aggregate}, now with the localization indicators included, gives}
\begin{equation}
 \mathbb E\Delta_{R,\ms}^2
 \le CA_n\log K\,\rho^2(\ms,\ms^0).
 \label{eq:variance-condition}
\end{equation}
\rev{Indeed, the localized analogue of $R_{\ms}(\mx_0,\my)$ is bounded above by $R_{\ms}(\mx_0,\my)$, while both localized losses are $O(A_n\log K)$. Their support restricts $\|\mx_0\|_\infty$ to $R_0=O(\sqrt{\log K})$, making the true-score loss bound in \eqref{eq:loss-truth-envelope} uniform there. Thus the variance condition holds with}
\begin{equation}
 c_v\asymp A_n\log K.
 \label{eq:cv}
\end{equation}
\rev{Both localized losses are nonnegative and bounded by $CA_n\log K$. Their difference therefore satisfies \eqref{eq:generic-boundedness}, with}
\begin{equation*}
 c_b\asymp A_n\log K,
\end{equation*}
\rev{after choosing fixed multiplicative constants. Since $A_n\ge1/2$, these choices give $c_v\ge1$ for large $K$ and $c_v+c_b\asymp A_n\log K$. By \eqref{eq:A-endpoints}, the latter is $O(\log^2K)$ for every admissible grid. The concentration bound keeps this dependence explicit; it does not treat these quantities as numerical constants.}

Let
\[
 \mathcal A_n:=\bigcap_{i=1}^n
 \{\|\mx_0^i\|_\infty\le R_0\}.
\]
By \eqref{eq:localization-pointwise} and the empirical optimality of $\widehat\ms$ for the full loss,
\begin{align*}
 P_n\ell_R(\widehat\ms)
 &\le P_n\ell(\widehat\ms)
 \le P_n\ell(\widetilde\ms)
 \le P_n\ell_R(\widetilde\ms)+r_K
\end{align*}
on $\mathcal A_n$. Furthermore,
\begin{equation}
 \mathbb P(\mathcal A_n^c)\le CnK^{-M}
 \label{eq:training-local-prob}
\end{equation}
by \eqref{eq:localization-prob} and a union bound. \rev{These observations verify the structural conditions of Lemma~\ref{lem:localized-erm} with $r:=r_K$ and $\eta:=\eta_K$. Its radius and covering conditions are checked next, before the localization failure probability is added.}

\subsection{Proof of Lemma~\ref{lem:grid-factor}}
\begin{proof}
	We first establish the lower bounds.
	Since \(\sigma_{t_k}^2\le1\),
	\[
	A_n\ge\sum_{k=1}^N h_k=T.
	\]
	Moreover, \(h_1=t_1\) and
	\(1-e^{-2t_1}\le2t_1\), so the first interval gives
	\[
	A_n\ge\frac{t_1}{1-e^{-2t_1}}\ge\frac12.
	\]
	Thus \(A_n\ge\max\{T,1/2\}\).
	
	For the upper bound, the function
	\(g(t):=(1-e^{-2t})^{-1}\) is decreasing on
	\((0,\infty)\). Therefore, for each \(k\ge2\),
	\[
	h_k g(t_k)
	\le\int_{t_{k-1}}^{t_k}g(t)\,\mathrm{d}t.
	\]
	Separating the first interval and summing yields
	\[
	\begin{aligned}
		A_n
		&\le\frac{t_1}{1-e^{-2t_1}}
		+\int_{t_1}^{T}\frac{\mathrm{d}t}{1-e^{-2t}}\\
		&=\frac{t_1}{1-e^{-2t_1}}
		+\frac12\log\frac{e^{2T}-1}{e^{2t_1}-1}.
	\end{aligned}
	\]
	The first term satisfies
	\[
	\frac{t_1}{1-e^{-2t_1}}
	=t_1+\frac{t_1}{e^{2t_1}-1}
	\le t_1+\frac12.
	\]
	Rewriting the logarithmic term consequently gives
	\[
	\begin{aligned}
		A_n
		&\le T+\frac12
		+\frac12\log
		\frac{1-e^{-2T}}{1-e^{-2t_1}}\\
		&\le T+\frac12
		+\frac12\log\frac{1}{1-e^{-2t_1}}.
	\end{aligned}
	\]
	By concavity on \([0,1]\) and monotonicity thereafter,
	\[
	1-e^{-2t_1}
	\ge(1-e^{-2})\min\{t_1,1\}.
	\]
	Hence
	\[
	\begin{aligned}
		A_n
		&\le T+\frac12
		+\frac12\log\frac{1}{1-e^{-2}}
		+\frac12\log^+(1/t_1)\\
		&\le C\bigl[1+T+\log^+(1/t_1)\bigr],
	\end{aligned}
	\]
	which proves \eqref{eq:A-endpoints}.
	
	Finally, the assumptions on the endpoints imply
	\[
	A_n\ge T\ge c_T\log K
	\]
	and
	\[
	\begin{aligned}
		A_n
		&\le C\bigl[1+(C_++C_-)\log K\bigr]\\
		&\le C'\log K,
	\end{aligned}
	\]
	where the last inequality uses \(K\ge e\).
	All estimates depend only on the endpoints and
	fixed constants, not on \(N\) or the interior grid
	points. This proves \(A_n\asymp\log K\).
\end{proof}

\subsection{Proof of Theorem~\ref{thm:score}}

\rev{\begin{proof}[Proof of Theorem~\ref{thm:score}]
The measurable minimizer in \eqref{eq:erm} is well defined. For each fixed $n$, the bounded-parameter, bounded-width, bounded-depth network class with at most $S$ nonzero parameters has a compact parameterization; the output-envelope restrictions are closed. The integrated loss is continuous in these parameters by dominated convergence, using the prescribed output bound and the Gaussian transition-score moments, and is measurable in the observations. Compactness and the measurable minimum theorem therefore give a measurable empirical minimizer.

We verify the approximation, localization, and finite-cover conditions separately. By Lemma~\ref{lem:approximation},
\begin{align*}
 \rho^2(\widetilde\ms,\ms^0)
 &\le C A_n K^{-2\beta/D}\log^{\beta+1}K,
\end{align*}
so the approximation radius is
\begin{equation}
 \delta_K
 :=C A_n^{1/2}K^{-\beta/D}\log^{(\beta+1)/2}K.
 \label{eq:deltaK}
\end{equation}

The covering condition requires $n\varepsilon^2$ to dominate $(c_v+c_b)$ times the logarithm of the cover size. Accordingly, define the stochastic scale by
\begin{equation}
 \zeta_K:=C\left(\frac{A_nK}{n}\right)^{1/2}\log^9K.
 \label{eq:zetaK}
\end{equation}
Choose the network size as in \eqref{eq:K-choice}, which balances the polynomial factors in $\delta_K$ and $\zeta_K$. To verify the cover condition for any $\varepsilon\ge C(\delta_K+\zeta_K)$, use monotonicity of covering numbers and set $u:=\min\{\varepsilon^2/128,1\}$. The bound $A_n\ge1/2$ gives
\[
 \varepsilon^2\gtrsim \frac{K}{n}\log^{18}K
 \asymp K^{-2\beta/D}\log^{18}K.
\]
Thus $u^{-1}$ is bounded by a fixed power of $K$. Lemma~\ref{lem:entropy} and the bound $B_K=O((\log K)^{3/2})$ imply
\[
 H_\infty(\varepsilon^2/128,\cG_R)
 \le H_\infty(u,\cG_R)
 \le CK\log^{17}K.
\]
On the other hand, $c_v+c_b\le CA_n\log K$ and the stochastic scale give
\[
 \frac{n\varepsilon^2}{c_v+c_b}
 \gtrsim \frac{n\varepsilon^2}{A_n\log K}.
\]
Taking the multiplicative constant in $\varepsilon\ge C(\delta_K+\zeta_K)$ sufficiently large makes this quantity exceed $80\{H_\infty(\varepsilon^2/128,\cG_R)+\log2\}$, because
$n\zeta_K^2/(A_n\log K)\asymp K\log^{17}K$.
This proves \eqref{eq:localized-entropy-condition}, with constants independent of $n$ and $N$.

For a prescribed tail exponent $M$, apply Lemma~\ref{lem:localization} with an internal exponent
$M_{\mathrm{loc}}>\max\{M,2\beta/D\}$.
Then $\eta_K+r_K\le CA_n\log K\,K^{-M_{\mathrm{loc}}}
=o(\delta_K^2+\zeta_K^2)$, and
$nK^{-M_{\mathrm{loc}}}\le nK^{-M}$.
The approximation bound and a further fixed enlargement of $\varepsilon$ verify \eqref{eq:localized-radius-condition}.

Both conditions therefore hold at the threshold in Theorem~\ref{thm:score},
\[
 \varepsilon_n=C A_n^{1/2}n^{-\beta/(D+2\beta)}\log^m n,
 \qquad m=\max\left\{9,\frac{\beta+1}{2}\right\},
\]
for sufficiently large fixed $C$. Adding the localization failure probability to \eqref{eq:localized-erm-tail} yields the stronger intermediate estimate
\[
 \mathbb P\{\rho(\widehat\ms,\ms^0)>\varepsilon_n\}
 \le2\exp\!\left\{-c\frac{n\varepsilon_n^2}{A_n\log K}\right\}
       +CnK^{-M}
 \le2\exp\{-c n^{D/(D+2\beta)}\log^{2m-1}n\}
       +CnK^{-M}.
\]
The tail stated in \eqref{eq:score-tail} is a convenient weaker consequence: since $A_n\ge1/2$ and $\log K\asymp\log n$, for sufficiently large $n$ the last exponential term is at most
\[
 4\exp\!\left\{
 -c\,\frac{n^{D/(D+2\beta)}\log^{2m-2}n}{A_n}\right\}.
\]
This establishes \eqref{eq:score-tail}, with the dependence on the growing variance and envelope constants made explicit.

Under the permitted approximation window, \eqref{eq:A-endpoints} and \eqref{eq:K-choice} give $A_n=O(\log K)=O(\log n)$. Substituting this bound into the score-error threshold yields the rate in \eqref{eq:score-op}, with logarithmic exponent $m+1/2$. The exponential remainder in \eqref{eq:score-tail} tends to zero, and choosing $M>(D+2\beta)/D$ also gives $nK^{-M}\to0$. Thus \eqref{eq:score-op} holds for every grid satisfying the theorem's approximation-window condition, without an additional assumption on $A_n$.
\end{proof}}

\rev{\section{Proofs of Lemmas~\ref{lem:chen-grid} and~\ref{lem:small-time-main}, Proposition~\ref{prop:sampling-bound}, and Theorem~\ref{thm:kl}}\label{sec:proof-kl}}

\rev{We first establish the reverse-process comparison, the small-time and remaining-step estimates, and the grid bounds. We then prove Lemma~\ref{lem:arbitrary-grid}(b), including the grid-independent transfer from the smoothed law to the original target for early stopping, and specialize it in Proposition~\ref{prop:sampling-bound}. The proof of Theorem~\ref{thm:kl} follows by inserting the score-estimation bound. Throughout this section, the training sample is fixed except when the final training-sample probability is invoked.}

\subsection{Proof of Lemma~\ref{lem:chen-grid}(a)}

Throughout this proof, \(u\) denotes reverse time and \(a=T-u\) denotes a forward-time cutoff. For use only in the appendix, write
\[
\begin{aligned}
 I_T&:=\mathbb E_{\my}\KLdiv\!\left(q_T\middle\|\mathcal N(\mathbf0,\mI_{d_x})\right),\\
 e_k&:=\mathbb E\|\widehat\ms_{t_k}(\mx_{t_k},\my)-\ms^0_{t_k}(\mx_{t_k},\my)\|_2^2,\\
 d_k(t)&:=\mathbb E\|\ms^0_t(\mx_t,\my)-\ms^0_{t_k}(\mx_{t_k},\my)\|_2^2,\\
 E(a)&:=\sum_{k:t_k>a}\bigl(t_k-\max\{t_{k-1},a\}\bigr)e_k,\\
 D(a)&:=\sum_{k:t_k>a}\int_{\max\{t_{k-1},a\}}^{t_k}d_k(t)\dd t.
\end{aligned}
\]
These expectations are under the true forward OU coupling and the marginal law of \(\my\), with the training sample fixed. In particular,
\[
 E(0)=\rho^2(\widehat\ms,\ms^0),\qquad
 E(t_1)=:\rho_{\mathrm{es}}^2\le\rho^2(\widehat\ms,\ms^0),\qquad
 D(0)=\int_0^{t_1}d_1(t)\dd t+D(t_1).
\]
All integrals starting at zero are over positive times. The comparison is understood in the extended-real sense; it does not require a separate moment or smoothness assumption on \(q_0\).

\begin{proof}
\textit{Step 1: finite initialization KL and positive-time comparison.}
If \(I_T=\infty\), the asserted inequality is immediate. Suppose therefore that \(I_T<\infty\). The variational inequality for relative entropy, applied conditionally to truncated versions of \(\|\mx\|_2^2/4\) and then averaged over \(\my\), gives
\[
 \mathbb E\|\mx_T\|_2^2
 \le4I_T+2d_x\log2<\infty.
\]
Here the Gaussian exponential moment is
\(\mathbb E_{Z\sim\mathcal N(\mathbf0,\mI_{d_x})}
 e^{\|Z\|_2^2/4}=2^{d_x/2}\).
The forward coupling \(\mx_T=e^{-T}\mx_0+\sigma_T Z\) consequently implies
\[
 \mathbb E\|\mx_0\|_2^2
 \le2e^{2T}\bigl(\mathbb E\|\mx_T\|_2^2+d_x\sigma_T^2\bigr)<\infty.
\]
Thus the conditional second moment is finite for almost every \(\my\). This is a consequence of finite initialization KL, not an additional assumption.

Fix such a \(\my\) and a positive forward cutoff \(a\). At positive times the OU densities are smooth and strictly positive. Apply the conditional transition comparison and KL chain rule from Proposition~C.3 of \citet{chen2023improved} on forward times \([a,T]\). Their OU process has drift \(-\mx/2\) and unit diffusion. Under the time change \(s=2t\), the drift difference in our normalization is twice the score difference and the diffusion coefficient is \(\sqrt2\); the integrated squared score discrepancy therefore has coefficient one. Averaging over \(\my\) yields
\begin{equation}
\begin{aligned}
 &\mathbb E_{\my}\KLdiv(q_a\|\widehat p_{T-a})\\
 &\quad\le I_T+
 \sum_{k:t_k>a}\int_{\max\{t_{k-1},a\}}^{t_k}
 \mathbb E\|\ms^0_t(\mx_t,\my)-\widehat\ms_{t_k}(\mx_{t_k},\my)\|_2^2\dd t\\
 &\quad\le I_T+2E(a)+2D(a).
\end{aligned}
\label{eq:KL-decomp}
\end{equation}
The last inequality inserts and subtracts
\(\ms^0_{t_k}(\mx_{t_k},\my)\) and uses
\(\|v+w\|_2^2\le2\|v\|_2^2+2\|w\|_2^2\).
For \(0<a<t_1\), the law \(\widehat p_{T-a}\) in this proof is obtained by running the last frozen-score step for time \(t_1-a\); this auxiliary cutoff is used only to pass to the full-sampler endpoint.

\textit{Step 2: the two output laws.}
At \(a=t_1\), equivalently \(u=T-t_1\), equation~\eqref{eq:KL-decomp} gives the sharper retained-step comparison
\begin{equation}
 \mathbb E_{\my}\KLdiv(q_{t_1}\|\widehat p_{T-t_1})
 \le I_T+2\rho_{\mathrm{es}}^2+2D(t_1)
 \le I_T+2\rho^2(\widehat\ms,\ms^0)+2D(0).
 \label{eq:retained-KL-comparison}
\end{equation}
This proves \eqref{eq:general-decomposition} for the early-stopped output. Its common right-hand side includes the omitted first interval, which only enlarges the bound.

For the full output, let \(a\downarrow0\). Continuity of the forward OU paths gives \(q_a(\cdot\mid\my)\Rightarrow q_0(\cdot\mid\my)\). The explicit frozen-score update similarly gives \(\widehat p_{T-a}(\cdot\mid\my)\Rightarrow\widehat p_T(\cdot\mid\my)\). Joint lower semicontinuity of relative entropy and Fatou's lemma imply
\[
\begin{aligned}
 \mathbb E_{\my}\KLdiv(q_0\|\widehat p_T)
 &\le\liminf_{a\downarrow0}
       \mathbb E_{\my}\KLdiv(q_a\|\widehat p_{T-a})\\
 &\le I_T+2\rho^2(\widehat\ms,\ms^0)+2D(0).
\end{aligned}
\]
The second inequality follows from \(E(a)\le E(0)\) and \(D(a)\le D(0)\). If \(D(0)=\infty\), the upper bound is interpreted accordingly. This proves the \(u=T\) case. No score or finite Fisher information at time zero, Gaussian-factor condition, or H\"older regularity has been used.
\end{proof}

\paragraph{Bounding initialization under the density model.}
The initialization term remains exact in Lemma~\ref{lem:chen-grid}(a). For the later rate calculation, Lemma~C.4 of \citet{chen2023improved} gives, for \(T\ge1\),
\begin{equation}
 I_T\le C(d_x+\mathfrak m^2)e^{-T},
 \qquad \mathfrak m^2:=\mathbb E\|\mx_0\|_2^2<\infty.
 \label{eq:init-bound}
\end{equation}
This follows by applying their estimate conditionally and averaging; a uniform-in-\(\my\) moment bound is unnecessary. The time rescaling only strengthens the exponential decay. Under the Gaussian-factor bounds in Lemma~\ref{lem:chen-grid}(b), and hence under Assumption~\ref{ass:density}, \(\mathfrak m^2\) is bounded by a constant depending only on \(B,C_2,d_x\), so \(I_T\lesssim e^{-T}\). We use this simplified estimate only under those additional assumptions.

\subsection{Proof of Lemma~\ref{lem:small-time-main} and Positive-Time Discretization Bounds}\label{sec:small-time-errors}

We prove both terms in Lemma~\ref{lem:small-time-main} under Assumption~\ref{ass:density}: first the score-increment estimate, then the chi-squared smoothing estimate. Both hold for every fixed $\beta>0$.

\rev{We separate the interval adjacent to the data from the positive-time discretization error. For $0<\beta\le1$, a bound on the time integral of the mean squared noisy score suffices; for $\beta>1$, the increment estimate of \citet{tian2026} gives the sharper smoothness-dependent power.}

\rev{\begin{lemma}[Small-time interval]\label{lem:first-interval}
Suppose Assumption~\ref{ass:density} holds with any fixed $\beta>0$. Then, uniformly over $\my$ and sufficiently small $t_1$,
\begin{equation*}
 \int_0^{t_1}
 \mathbb E
 \|\nabla\log q_t(\mx_t\mid\my)
 -\nabla\log q_{t_1}(\mx_{t_1}\mid\my)\|_2^2\dd t
 \le Ct_1^{r_\beta},
\end{equation*}
where $r_\beta=\min\{2,\beta\}$ and the expectation is under the conditional forward OU coupling. The integral is over $t>0$; no score at $t=0$ is required.
\end{lemma}}

\rev{\begin{proof}
For $\beta>1$, apply Lemma~17, equation~(51), of \citet{tian2026} with their smoothness parameter $r:=\beta$. The normalized extension and affine rescaling in Appendix~\ref{sec:proof-score} give their cube-domain density assumption uniformly in $\my$. Their exponent is $1+\alpha$ with $\alpha:=\min\{1,\beta-1\}$, hence $1+\alpha=r_\beta$.

It remains to prove the bound for $0<\beta\le1$. Fix $\my$. Under Assumption~\ref{ass:density}, the positive factor $f$ is uniformly $\beta$-H\"older in the data variable. At $\beta=1$, this follows from the bounded first partial derivatives and the mean value theorem along coordinate directions. In particular,
\[
 |f(\mx+\mathbf h,\my)-f(\mx,\my)|\le C\|\mathbf h\|_2^\beta.
\]
To exploit Gaussian smoothing, let $\varphi_v$ be the density of $\cN(\mathbf0,v\mI_{d_x})$, and set $v_0:=C_2^{-1}$ and $G:=(2\pi/C_2)^{d_x/2}$. Completing the square in the OU convolution gives
\[
 q_t(\mx\mid\my)=G\varphi_{v_t}(\mx)F_t(\mx,\my),
 \qquad F_t(\mx,\my):=\mathbb E f(b_t\mx+\tau_t\mathbf Z,\my),
 \quad \mathbf Z\sim\cN(\mathbf0,\mI_{d_x}),
\]
where the variance and smoothing parameters are
\[
 v_t:=\mu_t^2v_0+\sigma_t^2,\qquad
 b_t:=\frac{\mu_tv_0}{v_t},\qquad
 \tau_t^2:=\frac{v_0\sigma_t^2}{v_t}.
\]
For sufficiently small $t>0$, these satisfy $v_t\asymp1$, $b_t\asymp1$, and $\tau_t\asymp\sqrt t$, with constants depending on the fixed density parameters. Positivity gives $F_t\ge C_1$. Differentiating the Gaussian kernel, rather than $f$, is justified by boundedness of $f$ and integrability of the kernel derivative. Since $\mathbb E\mathbf Z=\mathbf0$, it yields
\[
 \begin{aligned}
 \nabla_{\mx}F_t(\mx,\my)
 &=\frac{b_t}{\tau_t}\mathbb E\!\left[\mathbf Z
       \{f(b_t\mx+\tau_t\mathbf Z,\my)-f(b_t\mx,\my)\}\right],\\
 \|\nabla_{\mx}F_t(\mx,\my)\|_2
 &\le Cb_t\tau_t^{\beta-1}\mathbb E\|\mathbf Z\|_2^{\beta+1}
 \le Ct^{(\beta-1)/2}.
 \end{aligned}
\]
The noisy score is therefore $\nabla\log q_t(\mx\mid\my)=-\mx/v_t+\nabla_{\mx}F_t(\mx,\my)/F_t(\mx,\my)$. The uniform conditional second-moment bound implies that its mean squared norm, denoted by $I_{\my}(t)$, satisfies
\[
 I_{\my}(t):=\mathbb E\!\left[\|\nabla\log q_t(\mx_t\mid\my)\|_2^2\mid\my\right]
 \le C(1+t^{\beta-1}).
\]
This bound may diverge at zero when $\beta<1$, but is integrable for every fixed $\beta>0$. Applying the squared-norm inequality to the score difference gives
\[
 \begin{aligned}
 &\int_0^{t_1}\mathbb E\!\left[
 \|\nabla\log q_t(\mx_t\mid\my)-\nabla\log q_{t_1}(\mx_{t_1}\mid\my)\|_2^2\mid\my\right]\dd t\\
 &\quad\le 2\int_0^{t_1}I_{\my}(t)\dd t+2t_1I_{\my}(t_1)
 \le C\left(t_1+\frac{t_1^\beta}{\beta}+t_1^\beta\right)
 \le C_\beta t_1^\beta.
 \end{aligned}
\]
The last step uses $t_1\le1$ and $\beta\le1$. Since $r_\beta=\beta$ in this range, the result follows. All constants are uniform in $\my$, but need not remain bounded as $\beta\downarrow0$.
\end{proof}}

\paragraph{Chi-squared smoothing: completion of Lemma~\ref{lem:small-time-main}.}
Suppose Assumption~\ref{ass:density} holds with $\beta>0$. Fix $\my$ and suppress it in the densities. Recall the Gaussian representation and parameters $v_0,G,v_t,b_t,\tau_t$ from the proof of Lemma~\ref{lem:first-interval}. That convolution identity is valid for every $\beta>0$, and gives
\[
 q_0(\mx)=G\varphi_{v_0}(\mx)f(\mx,\my),\qquad
 q_t(\mx)=G\varphi_{v_t}(\mx)F_t(\mx,\my),
\]
\[
 F_t(\mx,\my)=\mathbb E f(b_t\mx+\tau_t\mathbf Z,\my),
 \quad \mathbf Z\sim\cN(\mathbf0,\mI_{d_x}).
\]
The bounds on the smooth factor imply $C_1\le F_t\le B$. Moreover, $|v_t-v_0|+|b_t-1|\le Ct$ and $\tau_t^2\le Ct$ for small $t$. We bound $F_t-f$ separately in the two smoothness ranges, with constants uniform in $\my$.

\textit{The range $0<\beta\le1$.} For $0<\beta<1$, Assumption~\ref{ass:density} gives H\"older continuity in the data variable. For $\beta=1$, the bounded first partial derivatives give the same increment bound with exponent one, by the mean value theorem along coordinate directions. Thus, in both cases,
\[
 |f(\mx+\mathbf h,\my)-f(\mx,\my)|
 \le C\|\mathbf h\|_2^\beta.
\]
Apply this bound to the Gaussian displacement $\mathbf h=(b_t-1)\mx+\tau_t\mathbf Z$. Subadditivity of the power $\beta\le1$ and the finite Gaussian $\beta$-moment give
\[
 \begin{aligned}
 |F_t(\mx,\my)-f(\mx,\my)|
 &\le C\bigl(|b_t-1|^\beta\|\mx\|_2^\beta
              +\tau_t^\beta\mathbb E\|\mathbf Z\|_2^\beta\bigr)\\
 &\le C\bigl(t^\beta\|\mx\|_2^\beta+t^{\beta/2}\bigr)
 \le Ct^{\beta/2}(1+\|\mx\|_2^2).
 \end{aligned}
\]
Here $t\le1$ and $r_\beta=\beta$, so this is the required $t^{r_\beta/2}$ estimate without differentiating $f$ when $\beta<1$.

\textit{The range $\beta>1$.} The Taylor remainder in the data variable satisfies
\[
 |f(\mx+\mathbf h,\my)-f(\mx,\my)
   -\nabla_{\mx}f(\mx,\my)^\top\mathbf h|
 \le C\|\mathbf h\|_2^{r_\beta}.
\]
Here the bound follows from H\"older continuity of the gradient when $1<\beta<2$, and from bounded second derivatives when $\beta\ge2$. Apply it with $\mathbf h=(b_t-1)\mx+\tau_t\mathbf Z$. The Gaussian linear term has mean zero, so
\[
 |F_t(\mx,\my)-f(\mx,\my)|
 \le C\bigl(t\|\mx\|_2+t^{r_\beta}\|\mx\|_2^{r_\beta}
                 +t^{r_\beta/2}\bigr)
 \le Ct^{r_\beta/2}(1+\|\mx\|_2^2).
\]
In either smoothness range, therefore, $|F_t-f|\le Ct^{r_\beta/2}(1+\|\mx\|_2^2)$. To translate this estimate into a divergence bound, decompose the density difference as
\[
 q_0-q_t=G\{(\varphi_{v_0}-\varphi_{v_t})f+\varphi_{v_t}(f-F_t)\}.
\]
The chi-squared divergence is defined by $\chi^2(P\|Q):=\int(p-q)^2/q$ for densities $p,q$. The bounds $f\le B$ and $q_t\ge GC_1\varphi_{v_t}$, together with $(u+v)^2\le2u^2+2v^2$, give
\[
 \chi^2(q_0\|q_t)
 \le C\int\frac{(\varphi_{v_0}-\varphi_{v_t})^2}{\varphi_{v_t}}\dd\mx
    +C\int\varphi_{v_t}|f-F_t|^2\dd\mx.
\]
For sufficiently small $t$, the variance satisfies $2v_t>v_0$, and Gaussian integration gives the first integral as
\[
 \left(\frac{v_t^2}{v_0(2v_t-v_0)}\right)^{d_x/2}-1
 =O(t^2),
\]
while the second is $O(t^{r_\beta})$ by the preceding bound and uniformly bounded Gaussian fourth moments. Since $r_\beta\le2$, we obtain, uniformly in $\my$ and for sufficiently small $t>0$,
\begin{equation}
 \chi^2\!\left(q_0(\cdot\mid\my)\,\middle\|\,
 q_t(\cdot\mid\my)\right)\le Ct^{r_\beta}.
 \label{eq:smoothing-chi-square}
\end{equation}
This estimate is stronger than the corresponding smoothing bound in KL: the standard inequality $\KLdiv(P\|Q)\le\log\{1+\chi^2(P\|Q)\}\le\chi^2(P\|Q)$ follows, for example, from \citet[Eq.~(16)]{nishiyama2020relations}. Averaging \eqref{eq:smoothing-chi-square} over $\my$ and adding Lemma~\ref{lem:first-interval} proves Lemma~\ref{lem:small-time-main}.

For $t\ge t_1$, the standard score-increment estimate applies under a relative step-size condition in the normalization of \eqref{eq:forward}.

\begin{lemma}[Discretization after $t_1$]\label{lem:discretization}
Suppose the conditional second moment of $\mx_0$ is finite for almost every $\my$. There is a universal $c_\star>0$ such that, if
\begin{equation}
 \frac{h_k}{\sigma_{t_{k-1}}^2}\le\frac{c_\star}{d_x},
 \qquad k=2,\ldots,N,
 \label{eq:step-condition}
\end{equation}
then, for almost every $\my$, with a constant independent of $\my$ and expectation under the conditional forward OU coupling,
\rev{\begin{align}
 &\sum_{k=2}^N\int_{t_{k-1}}^{t_k}
 \mathbb E\!\left[
 \|\nabla\log q_t(\mx_t\mid\my)
 -\nabla\log q_{t_k}(\mx_{t_k}\mid\my)\|_2^2\mid\my\right]\dd t\nonumber\\
 &\qquad\le Cd_x^2\sum_{k=2}^N
 \frac{h_k^2}{\sigma_{t_{k-1}}^4}.
 \label{eq:disc-chen}
\end{align}}
\end{lemma}

\begin{proof}
Fix $\my$ with finite conditional second moment and apply Lemma~D.2, equivalently the discretization component of Theorem~2.2, of \citet{chen2023improved}, after the constant-factor time rescaling used in the proof of Lemma~\ref{lem:chen-grid}(a). The cited estimate requires only the relative step condition and a finite second moment, and its constant does not depend on that moment. This proves \eqref{eq:disc-chen}; averaging over $\my$ preserves the bound. The conditional moment condition holds in the finite-initialization case of Lemma~\ref{lem:chen-grid}(a), and in particular under Assumption~\ref{ass:density}. The estimate applies only for $k\ge2$, because $\sigma_{t_0}=0$ and Lemma~\ref{lem:first-interval} bounds the error over the small-time interval $[0,t_1]$ separately.
\end{proof}

\subsubsection{Arbitrary positive-time steps}\label{sec:arbitrary-step}
Consider the nontrivial case $I_T<\infty$ of Lemma~\ref{lem:chen-grid}(a), so that the preceding conditional moment condition holds. A quadratic step-size bound extends the discretization estimate to arbitrary positive-time steps, without a relative-step restriction.
For a positive-time interval \([a,a+h]\), let
\[
\mathcal D(a,h):=\int_a^{a+h}\mathbb E\|s_t^0(X_t,Y)-s_{a+h}^0(X_{a+h},Y)\|_2^2\dd t.
\]
When \(h/\sigma_a^2\le c_\star/d_x\), the single-interval estimate used in Lemma~\ref{lem:discretization} yields \(\mathcal D(a,h)\le Cd_x^2h^2/\sigma_a^4\). For a larger step, the conditional denoising identity and Jensen's inequality give \(\mathbb E\|s_t^0(X_t,Y)\|_2^2\le d_x/\sigma_t^2\), and hence
\[
\mathcal D(a,h)\le\frac{4d_xh}{\sigma_a^2}
\le\frac{4d_x^2h^2}{c_\star\sigma_a^4}.
\]
Thus $D(t_1)\le Cd_x^2\Pi$ with no relative-step condition. Substituting into the sharper retained-step comparison \eqref{eq:retained-KL-comparison} gives
\begin{equation}
\mathbb E_{\my}\KLdiv(q_{t_1}(\cdot\mid\my)\|\widehat p_{T-t_1}(\cdot\mid\my))
\le I_T+2\rho_{\mathrm{es}}^2+Cd_x^2\Pi.
\label{eq:early-smoothed-comparison}
\end{equation}
Under its additional Gaussian-factor assumptions, Lemma~\ref{lem:arbitrary-grid}(b) transfers this positive-time comparison to the original target. The replacement $I_T\lesssim e^{-T}$ is justified only when the density parameters, and hence the second moment in \eqref{eq:init-bound}, are controlled.

\subsection{Properties of the Nonuniform Time Discretization}\label{sec:grid-properties}

\rev{The statistical bound on $A_n$ already follows from the endpoints by \eqref{eq:A-endpoints}. The nonuniform discretization has a separate numerical purpose: geometric spacing near zero and uniform spacing away from zero enforce the relative step-size condition while controlling the discretization sum and the number of score evaluations. The following lemma records these properties together.}

\begin{lemma}[Grid bounds]\label{lem:grid}
For the recursive grid \eqref{eq:parameters}--\eqref{eq:grid-rule}, if $c_0$ is sufficiently small, then \eqref{eq:step-condition} holds and
\begin{equation}
 A_n=\sum_{k=1}^N\frac{h_k}{\sigma_{t_k}^2}
 \le C\left(1+\log\frac1{t_1}+T\right),
 \label{eq:A-grid}
\end{equation}
\begin{equation}
 \sum_{k=2}^N\frac{h_k^2}{\sigma_{t_{k-1}}^4}
 \le Cc_n\left(\log\frac1{t_1}+T\right),
 \label{eq:Pi-grid}
\end{equation}
and
\begin{equation}
 N\le Cc_n^{-1}
 \left(1+\log\frac1{t_1}+T\right).
 \label{eq:N-grid}
\end{equation}
\end{lemma}

\rev{\begin{proof}
For sufficiently large $n$, the calibrated cutoff has $t_1\le1$. The general endpoint estimate \eqref{eq:A-endpoints} immediately gives \eqref{eq:A-grid}, without using the geometric--linear structure. We therefore check only the numerical properties.

For a nonshortened geometric step, the rule gives
$t_{k-1}=(1-c_n)t_k$ and $h_k=c_nt_k$. Since these steps remain in $[0,1]$ and
$\sigma_t^2=1-e^{-2t}\asymp t$ on $[0,1]$,
\[
 \frac{h_k}{\sigma_{t_{k-1}}^2}
 \lesssim\frac{c_n}{1-c_n},
 \qquad
 \frac{h_k^2}{\sigma_{t_{k-1}}^4}
 \lesssim c_n^2.
\]
The number of geometric steps is
$O(c_n^{-1}\log(1/t_1))$, because
$\log(t_k/t_{k-1})=-\log(1-c_n)\asymp c_n$.

Every remaining nonshortened step starts from $t_{k-1}\ge1-c_n$. Taking $c_n\le1/2$, its forward variance is bounded below by a positive constant, $h_k=c_n$, and
$\sigma_{t_{k-1}}^2\asymp1$. Thus its relative step is $O(c_n)$ and its contribution to \eqref{eq:Pi-grid} is $O(c_n^2)$. The first such step crosses time one whenever $t_{k-1}<1$; there are $O((1+T)/c_n)$ steps in this regime. A shortened final step satisfies the same upper bounds and adds at most one to the count.

Combining the two regimes proves \eqref{eq:Pi-grid} and \eqref{eq:N-grid}; the additive constant in the discretization bound is absorbed by $\log(1/t_1)+T$ for the calibrated parameters and large $n$. The estimates also give $h_k/\sigma_{t_{k-1}}^2\le Cc_n/(1-c_n)$. Choosing $c_0$ sufficiently small establishes \eqref{eq:step-condition}. These numerical estimates agree with Lemma~D.3 of \citet{chen2023improved}, in the present indexing convention.

For completeness, we verify the order-optimality in \eqref{eq:Pi-order-optimal}, as discussed by \citet[Theorem~2.2 and Appendix~B]{chen2023improved}. Write $M=N-1$ for the number of retained intervals and $L=\log(1/t_1)+T$. For any grid with the same endpoints and $M$ retained intervals, Cauchy--Schwarz and the monotonicity of $t\mapsto1/\sigma_t^2$ give
\[
 \begin{aligned}
 \Pi
 &=\sum_{k=2}^N\left(\frac{h_k}{\sigma_{t_{k-1}}^2}\right)^2
 \ge\frac1M\left(\sum_{k=2}^N\frac{h_k}{\sigma_{t_{k-1}}^2}\right)^2\\
 &\ge\frac1M\left(\int_{t_1}^T\frac{\dd t}{1-e^{-2t}}\right)^2,
 \qquad
 \int_{t_1}^T\frac{\dd t}{1-e^{-2t}}
 =\frac12\log\frac{e^{2T}-1}{e^{2t_1}-1}\asymp L.
 \end{aligned}
\]
The last comparison holds for the calibrated endpoints and sufficiently large $n$. For the prescribed grid, counting geometric and linear steps also gives the matching order $M\asymp L/c_n$. Every nonshortened retained step has $h_k/\sigma_{t_{k-1}}^2\asymp c_n$, and the final shortened step has ratio at most $Cc_n$. Consequently,
\[
 \Pi\asymp Mc_n^2\asymp\frac{L^2}{M}
 \lesssim c_nL.
\]
This proves \eqref{eq:Pi-order-optimal}: the scheme attains, up to constants, the lower bound for grids with the same endpoints and retained-step count. It is an optimality statement for $\Pi$, not for the actual KL error.
\end{proof}}

\rev{\subsection{Proof of Lemma~\ref{lem:arbitrary-grid}(b)}\label{sec:proof-arbitrary-grid}
\begin{proof}
Fix the training sample and assume the Gaussian-factor bounds in Lemma~\ref{lem:arbitrary-grid}(b), together with its bounded, measurable fitted scores. No H\"older regularity is used in this proof. All constants are uniform in \(\my\) and independent of the interior grid and the fitted network, except through the displayed score bound.

\textit{Step 1: transfer under a bounded density ratio.} Let $P,Q,R$ be probability measures with densities and suppose $p/q\le H$ for a fixed constant $H>1$. The residual measure $S:=(HQ-P)/(H-1)$ is then a probability measure and satisfies $HQ=P+(H-1)S$. If $\KLdiv(Q\|R)=\infty$, the claimed upper bound is trivial. Otherwise, $p\le Hq$ and $s\le Hq/(H-1)$ ensure that all terms in the following identity are finite. Expanding the logarithms gives
\[
 \begin{aligned}
 \KLdiv(P\|R)+(H-1)\KLdiv(S\|R)
 &=H\KLdiv(Q\|R)+\KLdiv(P\|Q)\\
 &\quad +(H-1)\KLdiv(S\|Q).
 \end{aligned}
\]
Drop the nonnegative term involving $\KLdiv(S\|R)$ and apply the preceding KL--chi-squared inequality to the last two terms. Since $s-q=(q-p)/(H-1)$, this yields the transfer bound
\begin{equation}
 \KLdiv(P\|R)
 \le H\KLdiv(Q\|R)+\frac{H}{H-1}\chi^2(P\|Q).
 \label{eq:bounded-ratio-kl-transfer}
\end{equation}
This is a consequence of the mixture identity, not a triangle inequality. The density ratio $q_0/q_{t_1}$ need not be globally bounded when $C_2<1$. We therefore apply \eqref{eq:bounded-ratio-kl-transfer} on a large ball and control the omitted tail using a grid-independent lower bound on the sampler density.

\textit{Step 2: a sampler-density bound independent of the grid.}
Fix \(y\) and write \(U=T-t_1\). Consider the reference process
\[
 \mathrm dZ_u=Z_u\mathrm du+\sqrt2\,\mathrm dW_u,\qquad
 Z_0\sim\mathcal N(0,I_{d_x}),\qquad V_u=2e^{2u}-1.
\]
Let \(0=u_0<\cdots<u_m=U\) be the reverse partition corresponding to the
retained steps, and set
\(\theta_i=\sqrt2\,\widehat s_{T-u_i}(Z_{u_i},y)\).
The bound \(\left\|\theta_i\right\|_2\le\sqrt2 M_*\) ensures that
\[
 L=\exp\left\{I-\frac12\sum_i\left\|\theta_i\right\|_2^2\Delta u_i\right\},
 \qquad I=\sum_i\theta_i^\top\Delta W_i,
\]
is a valid Girsanov density. Under the measure with density \(L\), the
reference process follows the frozen-score interpolation of the sampler.
If \(r_y\) denotes the early-stopped output density, then
\begin{equation}
 \frac{r_y(x)}{\varphi_{V_U}(x)}=\mathbb E[L\mid Z_U=x].
 \label{eq:rewrite-endpoint-density}
\end{equation}

Set \(a_i=e^{U-u_i}\), \(b_i=e^{U-u_{i+1}}\), and \(v_i=a_i^2-1\).
Conditional Gaussian regression gives
\[
 \mathbb E[\Delta W_i\mid\mathcal F_{u_i},Z_U=x]
   =\frac{\sqrt2(a_i-b_i)}{v_i}(x-a_iZ_{u_i}).
\]
Given \(Z_U=x\), the residual on the right has mean \(v_ix/V_U\) and
covariance \(a_i^2V_{u_i}v_i I_{d_x}/V_U\preceq v_iI_{d_x}\). Since
\(\theta_i\) is \(\mathcal F_{u_i}\)-measurable, this implies
\[
 |\mathbb E[I\mid Z_U=x]|
 \le2M_*\left\{
 \frac{e^U-1}{V_U}\left\|x\right\|_2
 +\sqrt{d_x}\sum_i\frac{a_i-b_i}{\sqrt{a_i^2-1}}\right\}.
\]
The decreasing integrand gives the partition-independent estimate
\[
 \sum_i\frac{a_i-b_i}{\sqrt{a_i^2-1}}
 \le\int_1^{e^U}\frac{\mathrm dz}{\sqrt{z^2-1}}
 =\operatorname{arcosh}(e^U).
\]
Conditional Jensen's inequality in \eqref{eq:rewrite-endpoint-density} now
yields, almost everywhere,
\begin{equation}
 r_y(x)\ge\varphi_{V_U}(x)
 \exp\left\{-2M_*\frac{e^U-1}{V_U}\left\|x\right\|_2
 -2M_*\sqrt{d_x}\operatorname{arcosh}(e^U)-M_*^2U\right\}.
 \label{eq:rewrite-density-lower}
\end{equation}
Using \(V_U\ge1\), \(\log V_U\le2U+\log2\),
\((e^U-1)/V_U\le1\), and \(\operatorname{arcosh}(e^U)\le U+\log2\), together
with the upper Gaussian envelope for \(q_0\), gives
\begin{equation}
 \left[\log\frac{q_0(x\mid y)}{r_y(x)}\right]_+
 \le C(1+T)(1+M_*^2)(1+\left\|x\right\|_2^2).
 \label{eq:rewrite-log-density}
\end{equation}
In particular, this estimate contains neither the number of steps nor the
inverse variance of the last retained innovation.

\textit{Step 3: localization and restoration of the target tail.}
Take \(t=t_1\), \(P=q_0(\cdot\mid y)\),
\(Q=q_t(\cdot\mid y)\), and \(R=r_y\). Let
\(B_t=\{x:\left\|x\right\|_2^2\le1/t\}\), \(\delta=P(B_t^c)\), and
\(P_B=P(\cdot\mid B_t)\). The Gaussian envelope gives
\[
 \int_{B_t^c}(1+\left\|x\right\|_2^2)p(x)\dd x\le Ce^{-c/t}.
\]
For sufficiently small \(t\), we have \(\delta\le1/2\), and
the Gaussian convolution identity from the proof of Lemma~\ref{lem:first-interval}, which requires only bounded, measurable $f$, gives
\[
 \frac{q_0(x\mid y)}{q_t(x\mid y)}
 =\frac{\varphi_{v_0}(x)f(x,y)}{\varphi_{v_t}(x)F_t(x,y)}
 \le\frac{B}{C_1}\left(\frac{v_t}{v_0}\right)^{d_x/2}
 \exp\!\left\{\frac12\left(\frac1{v_t}-\frac1{v_0}\right)\|x\|_2^2\right\}
 \le C\exp(Ct\|x\|_2^2).
\]
Here $C_1\le F_t\le B$ and $|v_t-v_0|\le Ct$. Choose $\bar t_1\le1/2$ small enough that $2v_t>v_0$ and $\delta\le1/2$ for all $0<t\le\bar t_1$. This choice depends only on $C_1,C_2,B,d_x$. The Gaussian bounds then imply $\chi^2(P\|Q)<\infty$, and \(p_B\le Hq\) for a fixed \(H>1\) depending only on the same parameters.
Moreover,
\[
 \chi^2(P_B\|Q)
 \le\frac{1+\chi^2(P\|Q)}{(1-\delta)^2}-1
 \le4\chi^2(P\|Q)+8\delta
 \le4\chi^2(P\|Q)+Ce^{-c/t}.
\]
Apply \eqref{eq:bounded-ratio-kl-transfer} to \(P_B,Q,R\) and restore the
tail using the exact identity
\[
 \KLdiv(P\|R)=(1-\delta)\KLdiv(P_B\|R)
 +(1-\delta)\log(1-\delta)
 +\int_{B_t^c}p\log(p/r_y)\dd x.
\]
The logarithmic mass term is nonpositive. Bounding the tail integral with
\eqref{eq:rewrite-log-density} proves
\[
 \KLdiv(q_0(\cdot\mid y)\|r_y)
 \le C\KLdiv(q_t(\cdot\mid y)\|r_y)+C\chi^2(q_0(\cdot\mid y)\|q_t(\cdot\mid y))
       +C(1+T)(1+M_*^2)e^{-c/t}.
\]
Average over \(\my\) to obtain \eqref{eq:KL-early-general}, with the threshold $\bar t_1$ and constants specified in Lemma~\ref{lem:arbitrary-grid}(b). The proof retains the chi-squared discrepancy itself and does not invoke Lemma~\ref{lem:small-time-main} or any smoothness assumption. Combining with \eqref{eq:early-smoothed-comparison} further bounds the retained sampling error by initialization, score estimation, and discretization. The localization is only a proof
device and does not alter the sampler or introduce a training-sample event.

\end{proof}}

\rev{\subsection{Proof of Proposition~\ref{prop:sampling-bound}}\label{sec:proof-sampling-bound}
\begin{proof}
The setting of Theorem~\ref{thm:score} includes Assumption~\ref{ass:density}. Its Gaussian-factor bounds imply $I_T\lesssim e^{-T}$ by \eqref{eq:init-bound}. Lemma~\ref{lem:grid} controls $\Pi$ on the recursive grid, while Lemma~\ref{lem:small-time-main} gives
\[
 \int_0^{t_1}d_1(t)\dd t\le Ct_1^{r_\beta},
 \qquad
 \mathbb E_{\my}\chi^2(q_0\|q_{t_1})\le Ct_1^{r_\beta}.
\]
For the full sampler, Lemma~\ref{lem:chen-grid}(a) at reverse time $u=T$ and the decomposition of $D(0)$ give
\[
 \mathbb E_{\my}\KLdiv(q_0\|\widehat p_T)
 \lesssim e^{-T}+\rho^2(\widehat\ms,\ms^0)
                 +t_1^{r_\beta}+d_x^2\Pi.
\]
For early stopping, use the retained-step comparison \eqref{eq:early-smoothed-comparison}, corresponding to reverse time $u=T-t_1$, and then Lemma~\ref{lem:arbitrary-grid}(b). Since $\rho_{\mathrm{es}}^2\le\rho^2(\widehat\ms,\ms^0)$, this yields
\[
 \begin{aligned}
 \mathbb E_{\my}\KLdiv(q_0\|\widehat p_{T-t_1})
 &\lesssim e^{-T}+\rho^2(\widehat\ms,\ms^0)+t_1^{r_\beta}+d_x^2\Pi\\
 &\quad +(1+T)(1+M_*^2)e^{-c/t_1}.
 \end{aligned}
\]
The network envelope $M_t\lesssim\sqrt{\log K}/\sigma_t$, after conversion from maximum to Euclidean norm, implies $M_*^2\lesssim(\log n)/t_1$. Set $\lambda=2\beta/(D+2\beta)$ and $\nu=\lambda/r_\beta$. Under \eqref{eq:parameters},
\[
 (1+T)(1+M_*^2)e^{-c/t_1}
 \le C(1+\log n)(1+n^\nu\log n)e^{-cn^\nu}
 =o(n^{-\lambda})=o(t_1^{r_\beta}).
\]
For sufficiently large $n$, the cutoff also satisfies $t_1\le\bar t_1$ from Lemma~\ref{lem:arbitrary-grid}(b). Absorb the exponential remainder into the small-time term and substitute \eqref{eq:Pi-grid}. This proves \eqref{eq:schematic-decomposition} for both outputs, for each fixed training sample and fitted network in the prescribed class.

Finally, $K\asymp n^{D/(D+2\beta)}$, so the strict conditions $DC_->2\beta/r_\beta$ and $DC_+>2\beta$ ensure
\[
 K^{-C_-}\le t_1\le t_k\le T\le C_+\log K
\]
for all positive grid times and sufficiently large $n$. Thus the approximation window required by Theorem~\ref{thm:score} is respected. The endpoint bounds give
$T\le A_n\le C\{1+T+\log(1/t_1)\}$, hence
$A_n\asymp\log n\asymp\log K$.
\end{proof}}

\rev{\subsection{Proof of Theorem~\ref{thm:kl}}\label{sec:proof-kl-theorem}

\begin{proof}
With Proposition~\ref{prop:sampling-bound} established, it remains to insert the calibration and the score-estimation event. The choices \eqref{eq:parameters} give
\begin{equation}
 e^{-T}=n^{-\frac{2\beta}{D+2\beta}},
 \qquad
 t_1^{r_\beta}=n^{-\frac{2\beta}{D+2\beta}},
 \qquad
 c_n\left(\log\frac1{t_1}+T\right)
 =O\!\left(n^{-\frac{2\beta}{D+2\beta}}\log n\right).
 \label{eq:parameter-errors}
\end{equation}
Lemma~\ref{lem:grid} gives $A_n=O(\log n)$ and the step count \eqref{eq:N-rate}. Since $K\asymp n^{D/(D+2\beta)}$, the grid lies in the approximation range of Lemma~\ref{lem:approximation} for any fixed $C_->2\beta/(Dr_\beta)$ and $C_+>2\beta/D$, once $n$ is sufficiently large. Thus Theorem~\ref{thm:score} provides an event with probability at least \eqref{eq:kl-prob} on which
\begin{equation}
 \rho^2(\widehat\ms,\ms^0)
 \le Cn^{-\frac{2\beta}{D+2\beta}}\log^{2m+1}n.
 \label{eq:score-squared-final}
\end{equation}
On this event, substitute \eqref{eq:parameter-errors} and \eqref{eq:score-squared-final} into Proposition~\ref{prop:sampling-bound}. The score-error term dominates the initialization, small-time, and remaining discretization terms, giving the common target-KL bound \eqref{eq:kl-highprob}. The proposition's early-stopping transfer holds for every fitted network in the prescribed class, so no additional training-sample event is needed.

The logarithmic exponent satisfies $2m+1=\max\{19,\beta+2\}$. The full sampler uses $N$ score evaluations; omitting its final step uses $N-1$, with the same asymptotic order. This completes the proof.
\end{proof}}

%% file: introduction.tex
\section{Introduction}

\rev{Score-based diffusion models learn to reverse a process that gradually adds noise to data \citep{ho2020denoising,song2021score}. A neural network estimates the score, the gradient of the noisy log-density, and a numerical sampler uses this estimate to transform Gaussian noise into an observation \citep{vincent2011connection}. A guarantee for the generated distribution must therefore account for both learning and finite-step sampling. In conditional generation, each observation is paired with side information, such as a class label or covariates. The diffusion perturbs the observation while keeping its condition fixed. In this work, we study how the training sample size, conditional network complexity, and number of score evaluations jointly determine the accuracy of this procedure.}

\subsection{Prior Art}\label{sec:prior-art}

\paragraph{Score Learning and Statistical Guarantees.}
\rev{Score matching avoids the density normalizing constant \citep{Hyvarinen2005}, while denoising score matching provides an observable perturbation score for training \citep{vincent2011connection}. For unconditional diffusion, \citet{Oko2023} establish nearly minimax distribution-estimation rates over smoothness classes, and \citet{Zhang2024Minimax} remove density-lower-bound assumptions. Complementary analyses address neural training-sample complexity \citep{Gupta2024Training}, gradient-based optimization and generalization \citep{Han2024Score}, and optimal score estimation by empirical Bayes smoothing under sub-Gaussian tails and score regularity \citep{Wibisono2024Score}. These works clarify how data determine score accuracy. For a fully discrete conditional generator, that accuracy must be controlled at the particular times used by the numerical sampler. As such, we analyze a conditional score network fitted by minimizing a step-weighted denoising loss on the sampling grid, directly connecting its finite-sample error to the updates that generate an observation.}

\paragraph{Conditional Distribution Estimation.}
\rev{\citet{Li2022Conditional} provide a classical statistical benchmark for conditional density estimation, deriving minimax rates for a kernel-based estimator under smoothness in the generated variable and total-variation smoothness in the conditioning variable. Their results motivate accounting for regularity in both variables when analyzing conditional diffusion. In the diffusion setting, \citet{Fu2024UnveilCD} obtain smoothness-dependent TV rates under a positive Gaussian-factor density assumption, while \citet{Tang2025Conditional} allow different smoothness levels in the two variables and establish conditional TV rates under compact-support and density-lower-bound assumptions, together with manifold-adaptive Wasserstein guarantees. \citet{tian2026} establish conditional TV and square-root-KL guarantees under a positive Gaussian-factor model, including time-zero generation for smoothness greater than one. These statistical guarantees concern continuous-time conditional diffusion: in \citet{tian2026}, the reverse process is continuous before a final interval with a fixed score evaluation. In contrast, we provide a finite-sample guarantee for a fully discrete conditional generator, with learning and sampling analyzed on the same finite grid. This distinction is essential because an error bound integrated over diffusion time does not, by itself, control the score at the sampler's selected evaluation times. We retain the corresponding smoothness-dependent statistical exponent while explicitly accounting for every numerical update and the number of score evaluations. Moreover, our early-stopped guarantee controls the KL divergence from the original conditional data distribution, rather than only its smoothed version, for every $\beta>0$ (including the low-smoothness range $0<\beta\le1$).}

\paragraph{Sampling with Estimated Scores.}
\rev{Given an accurate score estimate, \citet{Lee2023General,chen2023sampling} establish convergence guarantees for diffusion samplers. \citet{chen2023improved} bound the KL divergence between the data distribution (or its Gaussian-smoothed version) and the generated distribution by three sources of error: initializing from a Gaussian rather than the exact terminal distribution, using an estimated score, and approximating the reverse process with finite time steps. Refined dimension and step bounds are obtained by \citet{Benton2024,LiYan2025}. \citet{Chen2026HighAccuracy} develop alternative algorithms with polylogarithmic dependence on inverse accuracy, given sufficiently accurate scores. These results quantify sampling error in terms of prescribed score accuracy, rather than the training sample size. \citet{Huang2026EndToEnd} connect learning and discrete sampling for unconditional residual score networks and Euler--Maruyama updates, including an explicit optimization gap. For conditional generation, an essential question is how to translate improved learning accuracy into improved output accuracy without letting numerical error dominate. We address this question by choosing the conditional network class and exponential-integrator grid together, so that finite-step sampling preserves the smoothness-dependent statistical rate. For a given number of training pairs, our bound specifies how many denoising steps suffice to attain this accuracy, with one score-network evaluation per step. The guarantee applies to both early-stopped and full sampling.}

\subsection{Motivation and Challenges}
\rev{Our motivation is to determine whether fully discrete conditional generation can retain the statistical accuracy of continuous-time models. From Section~\ref{sec:prior-art}, we see that prior work establishes smoothness-dependent statistical rates for continuous-time conditional diffusion \citep{Fu2024UnveilCD,Tang2025Conditional,tian2026}, while numerical sampling analyses bound the output error given a prescribed score accuracy \citep{chen2023improved}. The continuous-time guarantees do not account for discretization of the full reverse trajectory, and prescribed-score bounds do not establish the training sample size needed to achieve that accuracy. This motivates a conditional generation theory that connects finite training data to output accuracy and quantifies how many denoising steps suffice to preserve the statistical rate.}

\rev{Achieving this goal requires overcoming three challenges. Connecting learning to generation requires score-error control at the sampler's actual evaluation times, and an integrated-time bound alone does not ensure this control. The sampling grid must also keep numerical error within the statistical error bound. Finally, we must compare both early-stopped and full outputs with the original conditional distribution, rather than only a smoothed law, even in the low-smoothness range $0<\beta\le1$. These challenges lead to our central research question:}

\begin{center}
\rev{\emph{Can a fully discrete conditional diffusion model learned from training data attain smoothness-dependent accuracy relative to the original conditional data distribution with an explicit denoising-step count, while satisfying the same KL error bound with and without early stopping?}}
\end{center}

\subsection{Main Contribution}
\rev{We establish that fully discrete conditional diffusion can attain
smoothness-dependent statistical accuracy, with an explicit sampling budget
determined by the training sample size. For every smoothness index $\beta>0$, both early-stopped
and full sampling satisfy the same averaged conditional KL error bound
relative to the original conditional data distribution.
Specifically, our contribution is detailed as follows.} 
\begin{enumerate}
\item \emph{Learning guarantees for finite-step generation.}
\rev{We replace prescribed score accuracy with a finite-sample guarantee for
the conditional score actually learned on the sampling grid
(cf.\ Theorem~\ref{thm:score}). This resolves the mismatch between
integrated-time learning guarantees and finite-step generation, bringing
the score evaluations used by the sampler under direct statistical control.}

\item \emph{Guarantees for the original conditional data distribution.}
\rev{We establish KL error bounds against the original conditional data
distribution for both early-stopped and full sampling, covering every
$\beta>0$ (cf.\ Proposition~\ref{prop:sampling-bound}). In particular, our
early-stopping guarantee controls the unsmoothed conditional data distribution even in the
low-smoothness regime $0<\beta\le1$, rather than merely a Gaussian-smoothed
approximation.}

\item \emph{Statistical accuracy preserved under full discretization.}
\rev{We prove that numerical sampling need not degrade the
smoothness-dependent statistical accuracy of the learned model. Both
stopping rules achieve the same KL error bound and TV convergence rate,
and we specify the denoising-step budget sufficient to attain these
guarantees as the training sample grows
(cf.\ Theorem~\ref{thm:kl}).}
\end{enumerate}

\rev{\paragraph{Notation.}
We use $\|\cdot\|_2$ and $\|\cdot\|_\infty$ for the Euclidean and maximum norms of vectors, respectively, and $\langle\cdot,\cdot\rangle$ for the Euclidean inner product. The symbols $\mI_d$, $\cN(\mu,\Sigma)$, and $\ind\{E\}$ denote the $d$-dimensional identity matrix, the Gaussian law with mean $\mu$ and covariance $\Sigma$, and the indicator of an event $E$. For nonnegative quantities $U,V$, $U\lesssim V$ means $U\le CV$ for an absolute constant $C>0$ for fixed model parameters, and $U\asymp V$ means that both $U\lesssim V$ and $V\lesssim U$ hold. Constants $C,c>0$ may change between occurrences. Unless stated otherwise, they may depend on fixed model parameters but not on the sample size, network complexity, or time discretization.

For densities $p$ and $q$, we write
\[
 \operatorname{KL}(p\|q):=\int p(\mx)\log\frac{p(\mx)}{q(\mx)}\,\dd\mx,
 \qquad
 \operatorname{TV}(p,q):=\frac12\int|p(\mx)-q(\mx)|\,\dd\mx.
\]
For a positive deterministic sequence $b_n$, $Z_n=O_p(b_n)$ means that $Z_n/b_n$ is bounded in probability with respect to the training sample. The notation $\widetilde O$ or $\widetilde O_p$ suppresses factors that are powers of $\log n$.}

%% file: discussion.tex
\section{Discussion}\label{sec:conclusion}

\rev{We established finite-sample guarantees for fully discrete conditional diffusion by analyzing score learning and numerical sampling within a single framework. The resulting bounds relate the amount of training data to both generation accuracy and an explicit denoising step count. Under the stated density assumptions, the full and early-stopped samplers achieve the same averaged conditional KL error bound relative to the original conditional data distribution for every smoothness index $\beta>0$. Thus, the analysis shows how to preserve statistical accuracy from score estimation through finite-step generation.}

\rev{This framework opens several directions for further development. Incorporating optimization and artificial-noise Monte Carlo errors would extend the analysis to practical training procedures. Data-adaptive time grids and higher-order integrators could further reduce sampling effort at a prescribed statistical accuracy. Exploiting low-dimensional or factorizable conditional structure is another promising direction for sharpening the dependence on dimension.}